\documentclass[twoside,11pt]{article}
\usepackage{soul}
\usepackage[table]{xcolor}
\usepackage{color, xcolor}
\usepackage{blindtext}
\usepackage{amsmath,amsfonts}
\usepackage{array}
\usepackage{subfigure}
\usepackage{subcaption}
\usepackage{textcomp}
\usepackage{stfloats}
\usepackage{amsthm}
\usepackage[dvipsnames]{xcolor}
\usepackage{url}
\usepackage{bm}
\usepackage{algorithm}
\usepackage{algpseudocode}
\usepackage{verbatim}
\usepackage{booktabs}
\usepackage{graphicx}
\usepackage{multirow} 
\usepackage{color}
\usepackage{diagbox}
\usepackage{booktabs} 
\usepackage{multirow} 
\usepackage{array}    
\usepackage{caption}  
\usepackage{graphicx} 
\usepackage{siunitx}  
\usepackage{jmlr2e}
\usepackage{ulem}

\newcommand{\mylemmaname}{\bfseries Lemma}
\newcommand{\myassumptionname}{\bfseries Assumption}
\newcommand{\mydefinitionname}{\bfseries Definition}
\newcommand{\mypropositionname}{\bfseries Proposition}
\newcommand{\myclaimname}{\bfseries Claim}

\newtheorem{Definition}{\mydefinitionname} 
\newtheorem{Proposition}[theorem]{\mypropositionname} 

\ShortHeadings{Hypergraph Neural Diffusion: A PDE-Inspired Framework for Hypergraph Message Passing}{Zhiheng Zhou, Mengyao Zhou, Xixun Lin, Xingqin Qi and Guiying Yan.}
\firstpageno{1}

\begin{document}

\title{Hypergraph Neural Diffusion: A PDE-Inspired Framework for Hypergraph Message Passing}

\author{\name Zhiheng Zhou\textsuperscript{1,2}\footnotemark[1]\email zhouzhiheng@amss.ac.cn 
       \AND
       \name Mengyao Zhou\textsuperscript{2}\footnotemark[1]\email zhoumengyao@amss.ac.cn 
        \AND
       \name Xixun Lin\textsuperscript{3} \email linxixun@iie.ac.cn
       \AND
       \name Xingqin Qi \textsuperscript{1} \email qixingqin@sdu.edu.cn\\
       \name Guiying Yan\textsuperscript{2}\footnotemark[2] \email yangy@amss.ac.cn \\
              \addr \textsuperscript{1}School of Mathematics and Statistics, Shandong University\\
        \addr \textsuperscript{2}Academy of Mathematics and Systems Science,
       Chinese Academy of Sciences\\
       School of Mathematical Sciences,
       University of Chinese Academy of Sciences\\
       \addr \textsuperscript{3}Institute of Information Engineering,
       Chinese Academy of Sciences\\
       School of Cyber Security,
       University of Chinese Academy of Sciences}
\editor{My editor}
\maketitle
\footnotetext[1]{Equal contribution}
\footnotetext[2]{Corresponding author}

\begin{abstract}
Hypergraph neural networks (HGNNs) have shown remarkable potential in modeling high-order relationships that naturally arise in many real-world data domains. However, existing HGNNs often suffer from shallow propagation, oversmoothing, and limited adaptability to complex hypergraph structures. In this paper, we propose Hypergraph Neural Diffusion (HND), a novel framework that unifies nonlinear diffusion equations with neural message passing on hypergraphs. HND is grounded in a continuous-time hypergraph diffusion equation, formulated via hypergraph gradient and divergence operators, and modulated by a learnable, structure-aware coefficient matrix over hyperedge–node pairs. This partial differential equation (PDE) based formulation provides a physically interpretable view of hypergraph learning, where feature propagation is understood as an anisotropic diffusion process governed by local inconsistency and adaptive diffusion coefficient. From this perspective, neural message passing becomes a discretized gradient flow that progressively minimizes a diffusion energy functional. We derive rigorous theoretical guarantees, including energy dissipation, solution boundedness via a discrete maximum principle, and stability under explicit and implicit numerical schemes. The HND framework supports a variety of integration strategies such as non-adaptive-step (like Runge–Kutta) and adaptive-step solvers, enabling the construction of deep, stable, and interpretable architectures. Extensive experiments on benchmark datasets demonstrate that HND achieves competitive performance. Our results highlight the power of PDE-inspired design in enhancing the stability, expressivity, and interpretability of hypergraph learning.
\end{abstract}


\section{Introduction}
Hypergraphs offer a powerful mathematical framework for modeling complex systems with high-order relationships that go beyond pairwise interactions. In domains such as social networks~\citep{jia2021hypergraph,guan2023sparse,khan2025heterogeneous,su2025hy}, biological systems~\citep{ji2022fc,pan2024decgan,xie2024prediction,xia2024integration}, recommendation systems~\citep{li2022enhancing,peng2022gc,li2023next,yang2024self}, and image processing~\citep{wang2024hypergraph,wang2024ehgnn,zhang2025cross}, data naturally manifests as relations among groups of entities, which are more faithfully captured by hyperedges connecting multiple nodes simultaneously. To leverage such high-order structure for learning tasks, Hypergraph Neural Networks (HGNNs) have emerged as a natural extension of graph neural networks (GNNs), aiming to generalize message passing, aggregation, and representation learning to hypergraph domains.

While early HGNN models—such as HGNN~\citep{feng2019hypergraph}, HyperGCN~\citep{yadati2019hypergcn}, and HCHA~\citep{bai2021hypergraph}—approximate hypergraph convolutions via clique or star expansions, these approximations often distort the true combinatorial semantics of hyperedges. More expressive models~\citep{dong2020hnhn,huang2021unignn,chien2022you} adopt explicit node-to-edge-to-node message passing, yet often remain shallow and rigid in structure. Other directions—including attention-based designs~\citep{arya2020hypersage,choe2023classification}, simplicial or sheaf-based extensions~\citep{duta2023sheaf,choi2025hypergraph}, and dynamic system formulations~\citep{yan2024hypergraph}—have pushed the boundaries, but a unified and principled framework for adaptive, stable, and interpretable hypergraph learning remains elusive.

In parallel, diffusion-based frameworks rooted in Laplacians and partial differential equations (PDEs) have proven effective in modeling propagation dynamics, regularization, and smoothing in graphs~\citep{chamberlain2021grand,wang2023hypergraph,zheng2024co}. Extending this theory to hypergraphs, however, presents significant mathematical challenges due to the complex, non-pairwise structure of hyperedge–node interactions. In particular, there is a fundamental lack of well-defined notions of gradient and divergence operators on hypergraphs—constructs essential for expressing classical diffusion processes.

To address this, we introduce a novel definition of hypergraph gradient and divergence operators, grounded in physical intuition from diffusion theory. Specifically, the gradient operator quantifies the discrepancy between a node’s value and the average of its incident hyperedges, while the divergence aggregates these discrepancies back to nodes. This operator pair forms an adjoint system under standard inner products, naturally yielding a hypergraph Laplacian operator as their composition. Notably, we show that the corresponding Laplacian matrix recovers the normalized Laplacian of \cite{zhou2006learning}, but our derivation is grounded in variational principles and physical interpretability rather than combinatorial expansion.

Building on these operators, we propose Hypergraph Neural Diffusion (HND), a novel neural framework grounded in a continuous-time nonlinear hypergraph diffusion equation (HDE). This equation governs the evolution of node features as an anisotropic diffusion process modulated by a learnable, structure-aware coefficient matrix over hyperedge–node pairs. From this perspective, neural message passing becomes a discretized gradient flow that minimizes a diffusion energy functional—linking HGNN design to well-understood physical and mathematical principles.

HND offers several key advantages:
\begin{itemize}
    \item \textbf{Theoretically principled foundation:} HND is derived from a nonlinear PDE that generalizes classical diffusion dynamics to hypergraphs, enabling rigorous analysis of energy dissipation, stability, and boundedness.
    \item \textbf{Feature-adaptive and anisotropic propagation:} A learnable modulation matrix governs edge-dependent, data-aware information flow, enabling expressivity while providing adaptive control over smoothing.
    \item \textbf{Compatibility with numerical integration schemes:} The framework supports both fixed-step and adaptive-step integration methods, such as explicit and implicit Euler, multi-step methods, and adaptive-step solvers, each corresponding to stable and interpretable HGNN layers.
    \item \textbf{Extensive empirical validation:} HND achieves competitive performance on a wide range of node classification benchmarks, including multiple academic network and real-world dataset.
\end{itemize}
In sum, HND bridges the gap between PDE-based diffusion theory and neural hypergraph learning, offering a unified architecture that is expressive, stable, and interpretable. It not only advances the theoretical understanding of HGNNs but also provides a practical framework for building deep, adaptive, and structure-aware hypergraph neural networks.

\section{Related Work}
Graph diffusion techniques model the propagation of information over graphs and hypergraphs as continuous-time processes governed by Laplacian-like operators. Foundational works \citep{zhu2003semi,zhou2003learning} introduced diffusion-based semi-supervised learning via label propagation, which was later extended to hypergraphs through clique expansion and total variation formulations \citep{zhou2006learning,hein2013total}, allowing for higher-order smoothness control. More recent developments have advanced nonlinear and constrained diffusion frameworks \citep{tudisco2021nonlinear,tudisco2021anonlinear,prokopchik2022nonlinear}, enhancing the flexibility and expressiveness of diffusion processes. With the advent of deep learning, neural diffusion models \citep{chamberlain2021grand,li2022simple,gravina2023anti,wang2023equivariant,wang2023hypergraph} integrate diffusion dynamics into trainable architectures, offering improved robustness against over-smoothing and enabling feature-dependent, structure-aware propagation. Additionally, \cite{zheng2024co} propose CoNHD, which formulates ENC as a neural hypergraph diffusion process over co-representations of node–edge pairs. CoNHD introduces multi-input multi-output dynamics and adapts diffusion structures to the ENC structure, significantly improving expressivity and adaptability. \cite{choi2025hypergraph} propose Hypergraph Neural Sheaf Diffusion (HNSD), which builds symmetric simplicial sets from hyperedges and applies normalized sheaf Laplacians for diffusion. This approach generalizes classical Laplacians while preserving higher-order structure, offering a principled and geometrically grounded framework for hypergraph learning. stochastic and generative formulations \citep{gailhard2025hygene} reinterpret hypergraph diffusion through the lens of denoising diffusion models, paving the way for applications in generative modeling. Collectively, these advances position diffusion not only as a powerful modeling paradigm but also as a principled computational foundation for modern hypergraph learning.

Alongside diffusion-inspired approaches, HGNNs extend GNNs to model non-pairwise, higher-order relationships inherent in many real-world datasets. Early spectral HGNNs such as HGNN \citep{feng2019hypergraph}, HyperGCN \citep{yadati2019hypergcn}, and HCHA \citep{bai2021hypergraph} approximate hypergraph convolutions via clique or star expansion, often at the cost of losing the true structural semantics of hyperedges. To better preserve higher-order information, message-passing-based models like HNHN \citep{dong2020hnhn}, HyperSAGE \citep{arya2020hypersage,arya2024adaptive}, and UniGNN \citep{huang2021unignn} adopt a two-step aggregation scheme—first aggregating node features to hyperedges, then propagating back to nodes—providing a more faithful modeling of hypergraph topology. Meanwhile, architectures such as AllDeepSets and AllSetTransformer \citep{chien2022you} abandon spectral assumptions entirely, employing permutation-invariant set functions over hyperedges to facilitate flexible set-level reasoning. More recent advances, including WHATsNet \citep{choe2023classification} further advances edge-dependent processing by designing message-passing schemes conditioned on node–hyperedge information. \cite{duta2023sheaf} introduce Sheaf Hypergraph Networks, which endow hypergraphs with additional structure via cellular sheaves. They define both linear and nonlinear sheaf hypergraph Laplacians, extending classical diffusion frameworks. HDS \citep{yan2024hypergraph}, model hypergraph learning as a dynamic system using ordinary differential equations (ODEs), introducing controllability and stability to deep hypergraph propagation. To unify the modeling of node and edge semantics in hypergraphs, \citet{yan2024hypergraph1} propose a cross-expansion framework that maps both hypervertices and hyperedges to nodes in an expanded graph, enabling joint representation learning in a shared embedding space. To further improve the expressivity of HGNNs on long-range dependencies, \citet{xie2025k} propose K-hop Hypergraph Neural Network (KHGNN), which employs a novel bisection nested convolution module named HyperGINE. This module extracts features not only from nodes and hyperedges, but also from intermediate structural paths connecting them, effectively capturing multi-scale shortest-path interactions. In parallel, to mitigate the oversmoothing problem in deep HGNNs, \citet{li2025deep} propose FrameHGNN, a spectral HGNN framework built upon tight framelet transforms. FrameHGNN incorporates both low-pass and high-pass filters within the hypergraph convolution, enabling multifrequency information flow. The method is further enhanced with initial residual and identity mapping mechanisms, facilitating stable and expressive deep architectures. Collectively, these approaches reflect a growing emphasis on dynamic, expressive, and structure-aware neural computation for hypergraphs.

Despite the significant progress in hypergraph learning, existing models face several key limitations. Many traditional HGNNs rely on shallow message-passing schemes, often constrained by fixed propagation patterns (e.g., isotropic or uniform diffusion), which limits their ability to model heterogeneous or structure-dependent interactions across hyperedges. Spectral methods based on clique or star expansion may introduce redundancy or distort higher-order structure, while purely message-passing-based models often suffer from oversmoothing and insufficient depth scalability. Moreover, few existing methods explicitly incorporate mathematical principles from continuous dynamics (e.g., PDEs or ODEs), resulting in architectures that lack interpretability, stability guarantees, or fine-grained control over the information flow. Tasks such as require adaptive and asymmetric modeling over node–hyperedge pairs, remain particularly underexplored and poorly supported by most HGNN frameworks.

Motivated by the physical intuition of diffusion processes and the mathematical foundations of partial differential equations (PDEs), we introduce Hypergraph Neural Diffusion (HND)—a novel approach that bridges hypergraph learning and nonlinear diffusion theory. HND is grounded in a discretized nonlinear PDE on hypergraphs, where a learnable, structure-aware modulation matrix governs anisotropic and adaptive diffusion across hyperedge–node pairs. This enables fine-grained, feature-driven propagation beyond static Laplacians or uniform message passing. HND flexibly supports explicit and implicit schemes, high-order solvers (e.g., Runge–Kutta), allowing deeper and more stable architectures. Theoretically, HND preserves key PDE properties—such as energy dissipation and maximum principle—offering a principled and expressive alternative to conventional HGNNs for modeling higher-order relational data.

\section{Preliminaries and Notation}
This section introduces the mathematical foundations and notational conventions used throughout the paper. We begin by reviewing the structure of weighted hypergraphs and their associated matrix representations. We then define function spaces over nodes and hyperedge--node pairs, equipped with appropriate inner product structures. Building on these, we formalize the notions of gradient and divergence operators on hypergraphs, which generalize classical differential operators and form the basis for defining hypergraph Laplacians, energy functionals, and diffusion dynamics. These preliminaries provide the analytical framework upon which our diffusion models and neural architectures are constructed.
\subsection{Hypergraph Basics and Notation}
Let $\mathcal{H} = (\mathcal{V}, \mathcal{E}, \mathcal{W})$ be a weighted hypergraph, where $\mathcal{V}$ is the set of $n = |\mathcal{V}|$ nodes, $\mathcal{E}$ is the set of $m = |\mathcal{E}|$ hyperedges, and $\mathcal{W}: \mathcal{E} \to \mathbb{R}_{>0}$ assigns a positive weight $w_e$ to each hyperedge $e \in \mathcal{E}$. Each hyperedge $e$ is a subset of $\mathcal{V}$ with cardinality $|e| \geq 2$. The total number of hyperedge--node pairs is denoted as $N = \sum_{e \in \mathcal{E}} |e|$. We denote the set of all hyperedge--node pairs as $\mathcal{I} = \{(e, v) : v \in e, e \in \mathcal{E}\}$.

We now provide the explicit matrix form associated with the hypergraph structure. Let $H \in \mathbb{R}^{n \times m}$ be the incidence matrix of the hypergraph, where $H_{v,e} = 1$ if $v \in e$ and zero otherwise. Let $W_e \in \mathbb{R}^{m \times m}$ be the diagonal matrix of hyperedge weights, $D_e \in \mathbb{R}^{m \times m}$ the diagonal matrix of hyperedge degrees (i.e., $D_{e,e} = |e|$), and $D_v \in \mathbb{R}^{n \times n}$ the node degree matrix defined by $D_{v,v} = \sum_{e \ni v} w_e=d_v$.

\subsection{Function Spaces on Hypergraphs}
We define two real-valued function spaces associated with the hypergraph:
\begin{itemize}
    \item The node function space $L(\mathcal{V}) := \{f: \mathcal{V} \to \mathbb{R}\}\cong \mathbb{R}^n$. This space is equipped with the standard Euclidean inner product:
    \begin{equation}
    \langle f, f' \rangle_{L(\mathcal{V})} = \sum_{v \in \mathcal{V}} f(v) f'(v), \quad \forall f, f' \in L(\mathcal{V}).
    \end{equation}
    This inner product induces the norm $\|f\|_{L(\mathcal{V})} = \sqrt{\langle f, f \rangle_{L(\mathcal{V})}}$. Thus, $L(\mathcal{V})$ is a finite-dimensional Hilbert space.
    \item The hyperedge--node pair function space $L(\mathcal E, \mathcal V) := \{g: \mathcal{I} \to \mathbb{R}\}\cong \mathbb{R}^N$, defined over the set of pairs $(e,v)$ with $v \in e$. It is equipped with the standard inner product:
    \begin{equation}
    \langle g, g' \rangle_{L(\mathcal{E}, \mathcal{V})} = \sum_{(e,v) \in \mathcal{I}} w_eg(e,v) g'(e,v), \quad \forall g, g' \in L(\mathcal{E}, \mathcal{V}).
    \end{equation}
    This also defines a Hilbert space structure on $L(\mathcal{E}, \mathcal{V})$.
\end{itemize}
These two Hilbert spaces are connected via the gradient and divergence operators defined
below.

\subsection{Gradient and Divergence Operator on Hypergraphs}\label{subsection2.3}
To enable the formulation of differential operators and variational principles on hypergraphs, we introduce two key constructs: the \emph{gradient} and the \emph{divergence} operators. These operators generalize classical concepts from vector calculus to the hypergraph setting and serve as the foundational building blocks for defining Laplacian operators, energy functionals, and diffusion dynamics. The gradient operator measures the discrepancy between a node's value and the average over its incident hyperedges, capturing local inconsistency. Conversely, the divergence operator aggregates these edge-wise deviations to quantify the net flux at each node. Together, they form an adjoint pair under standard inner products, facilitating a rigorous analytical framework for hypergraph-based learning and diffusion models.
\begin{Definition}[Gradient Operator]
The \emph{gradient operator} is a linear map:
\[
\nabla: L(\mathcal{V}) \longrightarrow L(\mathcal{E}, \mathcal{V}),
\]
defined by
\begin{equation}\label{G1}
(\nabla f)(e, v): = \frac{f(v)}{\sqrt{d_v}} - \frac{1}{|e|} \sum_{u \in e}\frac{f(u)}{\sqrt{d_u}}.
\end{equation}
for every \( f \in L(\mathcal{V}) \) and \( (e,v) \in \mathcal{I} \). It measures the deviation of node \( v \) from the average value on hyperedge \( e \).
\end{Definition}
This quantity captures how much node $v$ deviates from the local equilibrium on hyperedge $e$, analogous to the classical directional gradient in Euclidean space.
\begin{Definition}[Divergence Operator]
The \emph{divergence operator} is defined as:
\[
\operatorname{div}: L(\mathcal{E}, \mathcal{V}) \longrightarrow L(\mathcal{V}),
\]
with
\begin{equation}\label{D2}
(\operatorname{div} g)(v) := \sum_{e \ni v}\frac{w_e}{\sqrt{d_v}}\left(  g(e, v) - \frac{1}{|e|} \sum_{u \in e} g(e, u)\right),
\end{equation}
for all \( g \in L(\mathcal{E}, \mathcal{V}) \) and \( v \in \mathcal{V} \). It represents the net flux imbalance at node \( v \), induced by the local flows along hyperedges.
\end{Definition}
This is consistent with the classical divergence interpretation: positive values indicate inflow, negative values indicate outflow, and zero indicates local conservation.

To formalize their connection, we now prove that the divergence operator is the adjoint of the gradient operator under the standard inner products.
\begin{Proposition}\label{P1}
The gradient operator $\nabla$ and the divergence operator $\operatorname{div}$ defined in Eq.(\ref{G1}) and (\ref{D2}) are adjoint with respect to the standard inner products on $L(\mathcal{V})$ and $L(\mathcal{E}, \mathcal{V})$:
\[
\langle \nabla f, g \rangle_{L(\mathcal{E}, \mathcal{V})} = \langle f, \operatorname{div} g \rangle_{L(\mathcal{V})}, \quad \forall f \in L(\mathcal{V}),\; g \in L(\mathcal{E}, \mathcal{V}).
\]
\end{Proposition}
The proofs of Propositions \ref{P1} can be found in Appendix \ref{app:proposition}. 

The above adjoint relationship establishes a fundamental duality between the gradient and divergence operators, analogous to the integration-by-parts principle in classical vector calculus. This property serves as the cornerstone for defining the Laplacian operator on hypergraphs.

\subsection{Laplacian Operator and Laplacian matrix on Hypergraph}
Motivated by the adjoint relationship established in Proposition~\ref{P1}, we define the hypergraph Laplacian operator $\Delta: L(\mathcal{V}) \to L(\mathcal{V})$ as the composition of divergence and gradient:
\begin{equation}
\Delta = \text{div} \circ \nabla.
\end{equation}
Specifically, by composing the divergence and gradient operators, we obtain a Laplacian operator that inherits key structural properties from this adjoint pair. In particular, the adjointness immediately implies that the resulting Laplacian operator is self-adjoint and induces a non-negative quadratic form, as formalized in follow Proposition.
\begin{Proposition}\label{P2}
Let the gradient operator $\nabla: L(\mathcal{V}) \to L(\mathcal{E}, \mathcal{V})$ and the divergence operator $\operatorname{div}: L(\mathcal{E}, \mathcal{V}) \to L(\mathcal{V})$ be defined as in Eq.(\ref{G1}) and (\ref{D2}), with $\Delta := \operatorname{div} \circ \nabla$. Then the hypergraph Laplacian operator $\Delta$ is self-adjoint (i.e., symmetric) and positive semi-definite with respect to the standard inner product on $L(\mathcal{V})$.
\end{Proposition}
Furthermore, this operator admits an explicit matrix representation, given in Proposition~\ref{P3}, which connects our formulation to existing normalized hypergraph Laplacians.
\begin{Proposition}\label{P3}
Let $\Delta = \operatorname{div} \circ \nabla$ be the hypergraph Laplacian operator defined as the composition of the divergence and gradient operators. Then its matrix representation is given by:
\begin{equation}
\mathcal{L} =  I - D_v^{-1/2}H W_e D_e^{-1} H^\top D_v^{-1/2},
\end{equation}
where $\mathcal{L}$ denotes the Laplacian matrix.
\end{Proposition}
The proofs of Propositions \ref{P2} and Propositions \ref{P3} can be found in Appendix \ref{app:proposition}. The matrix $\mathcal{L}$ is structurally identical to the normalized hypergraph Laplacian introduced by~\cite{zhou2006learning}:
\[
\mathcal{L}_{\text{Zhou}} = I - D_v^{-1/2} H W D_e^{-1} H^\top D_v^{-1/2},
\]
where \( W \) is the hyperedge weight matrix.

\paragraph{Comparison with Zhou's Laplacian.}

Although the matrix form of $\mathcal{L}$ is algebraically equivalent to the normalized hypergraph Laplacian in~\citep{zhou2006learning}, the two formulations arise from fundamentally different modeling perspectives and serve distinct purposes.

\cite{zhou2006learning} derive the hypergraph Laplacian through a spectral relaxation based on hypergraph-to-graph transformations, primarily targeting clustering and embedding tasks. In this formulation, the Laplacian is introduced as a fixed operator for downstream spectral analysis.

In contrast, our formulation constructs the Laplacian from first principles via hypergraph gradient and divergence operators, yielding a variational and operator-theoretic framework. This perspective enables (i) a direct connection to continuous-time diffusion dynamics, (ii) a natural extension to nonlinear or learnable operators through modulation mechanisms, and (iii) a new view of hypergraph neural networks as discretizations of underlying differential equations.

Therefore, while the resulting matrix expressions coincide in the linear case, our framework provides a more general and extensible foundation that supports dynamical modeling and principled neural architecture design beyond static spectral methods.

\subsection{Relation to Lovász-based Submodular Diffusion Operators.}
Recent works on submodular hypergraphs define nonlinear diffusion operators based on the Lovász extension of hyperedge cut functions \citep{li2018submodular,liu2021strongly,fountoulakis2021local}. In this framework, each hyperedge $e$ is associated with a submodular function $w_e(\cdot)$, whose Lovász extension $f_e(x)$ induces a nonlinear operator via its subdifferential:
\[
\frac{dx}{dt} \in - \sum_{e \in E} \partial f_e(x),
\]
where $\partial f_e(x)$ corresponds to a maximization over the base polytope.

In contrast, the proposed $\nabla/\mathrm{div}$ formulation defines a linear operator that can be written as a quadratic energy minimization:
\[
\frac{dx}{dt} = - \mathrm{div}(\nabla x),
\]
which corresponds to a normalized hypergraph Laplacian.

Importantly, our formulation can be viewed as a special case of the Lovász-based framework when the hyperedge cut functions are restricted to cardinality-based or quadratic forms, under which the Lovász extension reduces to a smooth quadratic function. In this case, the nonlinear subdifferential operator degenerates into a linear diffusion operator, recovering our $\nabla/\mathrm{div}$ formulation.

However, for general submodular hypergraphs, the Lovász-based operators are inherently nonlinear and set-valued, and therefore cannot be fully captured by our linear operator. Instead, our formulation provides a computationally efficient relaxation that preserves key diffusion characteristics while avoiding the complexity of submodular optimization.

This distinction highlights a trade-off between modeling flexibility and computational tractability, positioning our method as a scalable approximation to more general nonlinear hypergraph diffusions.

\section{Diffusion Equation on Hypergraph}
Diffusion processes are fundamental tools for modeling smoothness, information propagation, and dynamic evolution on structured domains. While classical diffusion equations have been extensively studied on Euclidean spaces and simple graph (a graph where each edge connects exactly two nodes.), extending such formulations to hypergraphs presents unique challenges due to their higher-order and non-pairwise relational structure. In this section, we introduce a nonlinear HDE that generalizes classical diffusion dynamics to hypergraph settings, capturing both the combinatorial complexity and feature-dependent anisotropy of real-world systems.

We begin by formulating the HDE using hypergraph gradient and divergence operators, which naturally encode local variations and flux across hyperedges. The modulation matrix governing diffusion strength is defined over hyperedge–node pairs, enabling the dynamics to adapt to evolving features. We then analyze key analytical properties of the HDE, including energy dissipation, solution well-posedness, and a discrete maximum principle, which collectively establish a rigorous foundation for stability and robustness. Finally, we discuss discretization methods and their implications for HGNNs, showing how classical numerical schemes translate into interpretable and theoretically grounded propagation mechanisms in deep architectures.

\subsection{Hypergraph Diffusion Equation}
We propose the following general nonlinear Hypergraph Diffusion Equation (HDE):
\begin{equation}\label{E1}
\frac{\partial \mathbf{x}(t)}{\partial t} = -\operatorname{div} \left[\mathbf{A}(\mathbf{x}(t)) \nabla \mathbf{x}(t) \right]
\end{equation}
with an initial condition $\mathbf{x}(0)$. Here, \(\mathbf{x}(t) \in L(\mathcal{V})\) is a time-dependent function defined on the node set, while \(\nabla\) and \(\operatorname{div}\) are the gradient and divergence operators defined previously. The operator \(\mathbf{A}(\mathbf{x}(t))\) is a modulation matrix acting on the hyperedge–node pair space, capturing adaptive diffusion strength. It is defined as follows.

\begin{Definition}[Modulation Matrix]
Let $\mathcal{H}=(\mathcal{V},\mathcal{E})$ be a hypergraph and 
$\mathcal{I} = \{(e, v) : v \in e, e \in \mathcal{E}\}$ be the hyperedge--node pairs. 
Define the modulation matrix $\mathbf{A}(\mathbf{x}(t))\in\mathbb{R}^{N\times N}$ as
\begin{equation}
    \mathbf{A}(\mathbf{x}(t))=\mathrm{diag}\big(a(\mathbf{x}_e(t),\mathbf{x}_v(t))\big),
\end{equation}
where each diagonal entry $a(\mathbf{x}_e(t),\mathbf{x}_v(t))>0$ is a similarity function defined for $(e,v)\in\mathcal{I}$, measuring the compatibility between hyperedge $e$ and node $v$. Moreover, for each node $v\in\mathcal{V}$, the similarity weights over all incident hyperedges are normalized as
\begin{equation}\label{A1}
\sum_{e\ni v} a(\mathbf{x}_e(t),\mathbf{x}_v(t)) = 1.
\end{equation}
\end{Definition}
The positivity condition $a(x_e, x_v) > 0$ ensures that all incident hyperedges contribute to the local aggregation at each node, preventing degenerate or disconnected diffusion behavior. Together with the normalization constraint, it guarantees that the modulation weights define a valid probability distribution, and that the resulting operator induces a well-behaved diffusion process. All modulation matrices considered in this section satisfy this definition.

Recall that the gradient operator $\nabla: L(\mathcal{V}) \to L(\mathcal{E}, \mathcal{V})$ and its adjoint divergence operator $\operatorname{div}: L(\mathcal{E}, \mathcal{V}) \to L(\mathcal{V})$ are both linear maps. Let $P \in \mathbb{R}^{N \times n}$ and $Q\in\mathbb{R}^{n \times N}$ be the matrix representation of the gradient operator and divergence operator such that:
\begin{equation}
\nabla f = P f,
\quad \text{and} \quad 
\operatorname{div}(g) = Q g,
\end{equation}
where $f\in L(\mathcal{V})$ is a function over the node set and $g \in L(\mathcal{E}, \mathcal{V})$ is a function over the hyperedge--node pairs. From Eq.(\ref{PQ}) in Appendix \ref{proof3}, we have $P=(B-C)D_v^{-1/2}$ and $Q = D_v^{-1/2}(B-C)^\top S$, let $G=S^{-1/2}(B-C)D_v^{-1/2}$, then $PQ=G^TG$.

Given that the modulation matrix $\mathbf{A}(\mathbf{x}(t))$ is a diagonal matrix of shape $N \times N$, because the $\mathbf{A}(\mathbf{x}(t))$ is a diagonal matrix, then we can rewrite the Eq.(\ref{E1}) in purely matrix form:
\begin{equation}\label{L1}
\frac{\partial \mathbf{x}(t)}{\partial t} 
= -\operatorname{div} \left[ \mathbf{A}(\mathbf{x}(t), t) \nabla \mathbf{x}(t) \right]
=-Q\mathbf{A}(\mathbf{x}(t))P\mathbf{x}(t)= -G^\top \mathbf{A}(\mathbf{x}(t)) G \mathbf{x}(t).
\end{equation}

Therefore, the time evolution of the node function $\mathbf{x}(t)$ under hypergraph-structured diffusion is governed by a state-dependent Laplacian matrix:
\begin{equation}
\frac{\partial \mathbf{x}(t)}{\partial t} = -\mathcal{L}_{\text{NL}}(\mathbf{x}(t)) \mathbf{x}(t),
\quad \text{where} \quad
\mathcal{L}_{\text{NL}}(\mathbf{x}(t)) = G^\top \mathbf{A}(\mathbf{x}(t)) G.
\end{equation}

This form generalizes classical graph diffusion by incorporating higher-order relations (via hyperedges) and nonlinear, data-driven anisotropy (via $\mathbf{A}(\mathbf{x}(t))$). Hypergraph diffusion dynamics also offer insight into the architecture of HGNNs, this perspective allows the diffusion strength to adapt dynamically to node and hyperedge embeddings, providing a principled foundation for designing interpretable and robust message-passing mechanisms. The entries of $\mathbf{A}(\mathbf{x}(t))$ may be parameterized via neural attention modules offering compatibility with neural architectures.

\subsection{Analytical Properties of Nonlinear HDE}
The proposed nonlinear HDE in Eq.(\ref{E1}) exhibits several rich analytical properties. These properties reflect its suitability as a diffusion framework that incorporates higher-order combinatorial structure, state-adaptive anisotropy, and numerical stability. We present a detailed exploration of these properties and elucidate their mathematical and physical significance. Furthermore, we highlight their implications in the design and analysis of HGNNs, thereby establishing theoretical foundations for architecture stability, expressivity, robustness, and interpretability.

We begin by analyzing the fundamental energy dissipation behavior, which underpins the stability and convergence of the system. This is followed by an investigation of the solution's existence and uniqueness, ensuring the model is well-posed. Finally, we derive the discrete maximum principle to establish boundedness of the solution over time.

\paragraph{Energy Dissipation and Monotonicity.}
We define the diffusion energy functional as
\begin{equation}\label{E2}
\mathcal{E}(\mathbf{x}(t)) := \frac{1}{2} \langle \mathbf{x}(t), \mathcal{L}_{\text{NL}}(\mathbf{x}(t)) \mathbf{x}(t) \rangle = \frac{1}{2} (G \mathbf{x}(t))^\top \mathbf{A}(\mathbf{x}(t)) (G \mathbf{x}(t)) = \frac{1}{2} \| G \mathbf{x}(t) \|^2{\mathbf{A}(\mathbf{x}(t))},
\end{equation}
where $\mathbf{A}(\mathbf{x}(t)) = \operatorname{diag}(a(\mathbf{x}_e(t), \mathbf{x}_v(t)))$ is the modulation matrix. The negative sign reflects the fact that the diffusion process is associated with a loss of potential energy over time. In physical terms, this energy functional quantifies the total anisotropic gradient energy in the system, with each edge-node pair weighted by $a(\cdot,\cdot)$.

Importantly, the evolution equation of the nonlinear HDE in Eq.(\ref{L1}) can be interpreted as a gradient flow of the energy functional $\mathcal{E}(\mathbf{x}(t))$ with respect to the standard Euclidean inner product:
\begin{equation}
\frac{\partial \mathbf{x}(t)}{\partial t} = - \nabla_{\mathbf{x}(t)} \mathcal{E}(\mathbf{x}(t)).
\end{equation}
We adopt the standard approximation that treats the modulation matrix $\mathbf{A}(\mathbf{x}(t))$ as locally fixed when computing the gradient, in order to preserve the variational gradient flow structure. A full derivative would involve higher-order terms from the dependence of $\mathbf{A}(\mathbf{x}(t))$ on $\mathbf{x}(t)$, which are neglected here for analytical tractability.

That is, the system evolves in the direction of steepest descent of the energy landscape. This formulation provides a principled variational perspective: the dynamics seek to minimize the total inconsistency of node values within hyperedges in an optimal way, progressively aligning local node values with their respective hyperedge means.
\begin{Proposition}[Energy Dissipation]\label{P4}
Let $\mathbf{x}(t)$ be a differentiable solution of the nonlinear HDE. Then the energy $\mathcal{E}(\mathbf{x}(t))$ in Eq.(\ref{E2}) is monotonically decreasing:
\[
    \frac{d}{dt} \mathcal{E}(\mathbf{x}(t)) = - \| \frac{\partial \mathbf{x}(t)}{\partial t} \|^2 \leq 0.
\]
\end{Proposition}

\begin{proof}
By definition,
\begin{equation}
\frac{d}{dt} \mathcal{E}(\mathbf{x}(t)) = \left\langle \nabla_{\mathbf{x}} \mathcal{E}(\mathbf{x}(t)), \frac{\partial \mathbf{x}(t)}{\partial t} \right\rangle = \left\langle G^\top \mathbf{A}(\mathbf{x}) G \mathbf{x}, \frac{\partial \mathbf{x}(t)}{\partial t} \right\rangle.
\end{equation}
Using the evolution equation $\frac{\partial \mathbf{x}(t)}{\partial t} = -G^\top \mathbf{A}(\mathbf{x}) G \mathbf{x}$,
\begin{equation}
\frac{d}{dt} \mathcal{E}(\mathbf{x}(t)) = - \left\langle \frac{\partial \mathbf{x}(t)}{\partial t}, \frac{\partial \mathbf{x}(t)}{\partial t} \right\rangle = - \| \frac{\partial \mathbf{x}(t)}{\partial t} \|^2 \leq 0.
\end{equation}
\end{proof}

The monotonic decrease of this energy indicates that the system continually dissipates potential energy and evolves toward a locally minimal configuration. This guarantees that the dynamics are inherently stable and non-oscillatory, preventing divergence or chaotic behavior. From an optimization perspective, it implies that the nonlinear HDE defines a descent process over the energy landscape, with each trajectory following the direction of maximal local reduction in inconsistency. Consequently, the long-term behavior of the system is governed by the landscape of $\mathcal{E}(\mathbf{\mathbf{x}(t)})$, leading the state toward a steady equilibrium that balances hyperedge consistency and node-wise variation.

Physically, this behavior reflects a relaxation process in which the system dissipates potential energy, analogous to heat diffusion in a non-uniform medium where conductivity varies spatially. The modulation matrix $\mathbf{A}(\mathbf{x}(t))$ plays the role of a position-dependent conductance, governing the rate and directionality of diffusion. This introduces anisotropy and heterogeneity into the dynamics, allowing the system to adapt to non-homogeneous structural and functional properties of the hypergraph.

Within HGNN frameworks, this energy minimization perspective aligns with the objective of smoothing node features across hyperedges while preserving sharp transitions guided by learned attention. As such, it helps mitigate the oversmoothing issue prevalent in deep GNN architectures by enabling state-adaptive diffusion strength that limits over-homogenization of features across structurally diverse regions.
\paragraph{Existence and Uniqueness of Solutions.} The monotonic dissipation of energy suggests that the system evolves in a stable and structured manner. A natural question that follows is whether the diffusion dynamics admit a well-defined solution trajectory for any initial condition. To address this, we examine the regularity properties of the vector field that governs the right-hand side of the diffusion equation. In particular, we assess whether the system satisfies classical conditions for local existence and uniqueness in the context of nonlinear differential equations.

\begin{Proposition}[Local Existence and Uniqueness]\label{P5}
Suppose the modulation matrix $\mathbf{A}(\mathbf{\mathbf{x}(t)})$ is continuous and locally Lipschitz in $\mathbf{\mathbf{x}(t)}$. Then, for any initial condition $\mathbf{x}(0) \in \mathbb{R}^n$, there exists a unique local solution $\mathbf{x}(t) \in C^1([0, T], \mathbb{R}^n)$ to the nonlinear HDE in Eq.(\ref{L1}).
\end{Proposition}

\begin{proof}
We start by defining the vector field \( F(\mathbf{x}(t)) \) as:

\begin{equation}
F(\mathbf{x}(t)) = - G^\top \mathbf{A}(\mathbf{x}(t)) G \mathbf{x}(t)
\end{equation}

The modulation matrix \( \mathbf{A}(\mathbf{x}(t)) \) is assumed to be continuous and locally Lipschitz in \( \mathbf{x}(t) \). Since \( F(\mathbf{x}(t)) \) is a composition of smooth functions, it is also locally Lipschitz. Specifically, there exists a constant \( L \) such that for any \( \mathbf{x}_1(t) \) and \( \mathbf{x}_2(t) \), we have:

\begin{equation}
\| F(\mathbf{x}_1(t)) - F(\mathbf{x}_2(t)) \| \leq L \| \mathbf{x}_1(t) - \mathbf{x}_2(t) \|
\end{equation}

This guarantees that the vector field does not change too rapidly, ensuring that the solutions to the equation are controlled.

The Picard--Lindelöf theorem \citep{lindelof1894application} guarantees the existence and uniqueness of solutions to the initial value problem for first-order ODEs, provided the vector field is locally Lipschitz. 
Since \( F(\mathbf{x}(t)) \) is locally Lipschitz, the Picard-Lindelöf theorem guarantees a unique solution to the nonlinear HDE on some time interval \( [0, T] \). Therefore, there exists a unique solution \( \mathbf{x}(t) \in C^1([0, T], \mathbb{R}^n) \) to Eq. (\ref{L1}) for the given initial condition \( \mathbf{x}(0) \).
\end{proof}

The local solution may be extended globally in time as long as the state remains bounded. The energy dissipation property prevents blow-up, and the boundedness of $\mathcal{E}(\mathbf{x}(t))$ implies that $\|G \mathbf{x}(t)\|$ and hence $\|\mathbf{x}(t)\|$ remain controlled. In particular, under the discrete maximum principle, $\|\mathbf{x}(t)\|$ remains bounded for all $t$, ensuring global well-posedness of the system. This analytical guarantee can be leveraged in HGNN models to justify their stability under repeated message passing layers and dynamic attention modulation. In practical terms, it offers a theoretical basis for designing deep hypergraph architectures without the risk of uncontrolled feature amplification.

\paragraph{Maximum Principle and Boundedness.}
One of the fundamental properties of diffusion processes is that they tend to preserve the range of initial values over time. For the nonlinear HDE, this intuition is formalized through a discrete maximum principle, which guarantees that the solution remains bounded within the initial extrema. This result plays a central role in establishing both numerical stability and global existence of solutions.
\begin{Proposition}[Discrete Maximum Principle]\label{P6}
Let $\mathbf{x}(t)$ be a solution to the nonlinear HDE in Eq.(\ref{L1}) with initial condition $\mathbf{x}(0)$. Then for all $t \geq 0$, the solution satisfies:
\begin{equation}
\min_{v \in \mathcal{V}} \frac{\mathbf{x}_v(0)}{\sqrt{d_v}} \leq \frac{\mathbf{x}_v(t)}{\sqrt{d_v}} \leq \max_{v \in \mathcal{V}} \frac{\mathbf{x}_v(0)}{\sqrt{d_v}}.
\end{equation}
That is, the evolution remains pointwise bounded by the initial range of values.
\end{Proposition}
The proofs of Propositions \ref{P6} can be found in Appendix \ref{app:proposition}.

This principle provides a powerful control on the dynamics by ruling out divergence and ensuring that the system remains well-behaved over time. Combined with energy dissipation, it implies that the solution remains both bounded and convergent. In numerical implementations, the discrete maximum principle serves as a guarantee for stability and interpretability of the learned representations. In HGNNs, such boundedness corresponds to robustness of the learned embeddings and prevents feature explosion during message passing. It also contributes to model interpretability by ensuring that node features remain within semantically meaningful ranges throughout training. The convexity-driven interpretation provides further justification for the use of normalized attention or aggregation functions in hypergraph-based architectures. From both theoretical and practical perspectives, these properties ensure that nonlinear HDEs offer a principled and stable backbone for HGNN design and analysis.

\subsection{Discretization of HDE}
Recall HDE in Eq.(\ref{L1}):
\begin{equation}\label{E3}
\frac{\partial \mathbf{x}(t)}{\partial t} = - G^\top \mathbf{A}(\mathbf{x}(t)) G \mathbf{x}(t),
\end{equation}
where \( G \) is the gradient matrix of the hypergraph, and $\mathbf{A}(\mathbf{x}(t)) = \operatorname{diag}(a(\mathbf{x}_e(t), \mathbf{x}_v(t)))$ is the modulation matrix. For numerical solution, we need to discretize Eq.(\ref{E3}) in time.

From the perspective of HGNNs, the HDE provides a continuous-time framework for understanding information diffusion across hypergraph structures. Specifically, the gradient-based flow governed by the HDE naturally corresponds to the feature propagation mechanisms central to HGNNs, where information is aggregated from hyperedges to nodes in a structure-aware manner. By discretizing the HDE in time, one obtains update rules that resemble those used in message-passing HGNN architectures. This connection bridges differential equations and neural computation, offering a principled foundation for designing HGNN models with improved stability, expressiveness, and interpretability.

The temporal discretization of the HDE can be approached using two main classes of methods: explicit and implicit schemes. These methods differ in how the state is updated based on the current and previous time steps. The explicit method uses the current state, whereas the implicit method uses future state information, offering distinct advantages in terms of stability, especially for training HGNNs.
\paragraph{Explicit Euler Method}
The Explicit Euler method is a straightforward time integration technique that approximates the time derivative using the current state of the system. The discretized form of the HDE using the explicit Euler method is given by:
\begin{equation}\label{eluere}
\mathbf{x}^{(k+1)} = \mathbf{x}^{(k)} - \tau \cdot G^\top \mathbf{A}(\mathbf{x}^{(k)}) G \mathbf{x}^{(k)},
\end{equation}
where \( \tau \) is the time step. This method updates the state based on the information available at the current time step, making it computationally simple and easy to implement. However, the stability of the explicit method can be sensitive to the choice of \( \tau \), particularly when the time step is large, potentially leading to numerical instability.

The stability of the explicit Euler method is generally verified using Von Neumann stability analysis. For the HDE, stability requires that the norm of the state does not grow with each update. Specifically, the condition:
\begin{equation}
\| \mathbf{x}^{(k+1)} \| \leq \| \mathbf{x}^{(k)} \|,
\end{equation}
must hold, ensuring that the energy of the system is non-increasing.  In numerical computation, the stability of the explicit method is limited by the step size, with larger step sizes potentially leading to numerical instability. To establish the stability condition more rigorously, we analyze the behavior of the method and derive the necessary conditions for stability.
\begin{Proposition}[Explicit Euler Stability]\label{P7}
Let \( \mathbf{A}(\mathbf{x}) \) be the modulation matrix at time \( t_k \). Then the explicit Euler method for the HDE in Eq.(\ref{eluere}) is stable if the step size \( \tau \leq 1 \).
\end{Proposition}
\paragraph{Implicit Euler Method}
The Implicit Euler method is an implicit integration technique that uses information from the future time step to update the current state. It is generally more stable than the explicit method, especially when dealing with stiff equations. For the hypergraph diffusion equation, the discretization using the Implicit Euler method is given by:
\begin{equation}\label{elueri}
\mathbf{x}^{(k+1)} = \mathbf{x}^{(k)} - \tau \cdot G^\top \mathbf{A}(\mathbf{x}^{(k+1)}) G \mathbf{x}^{(k+1)}.
\end{equation}
where $\tau$ is the time step. In the Implicit Euler method, the future state $\mathbf{x}^{(k+1)}$  is used to update the current state $\mathbf{x}^{(k)}$, which makes it more stable compared to the explicit method.

The Implicit Euler method typically exhibits excellent stability properties since it is an $A$-stable method, meaning it remains stable for all eigenvalues. This stability is particularly beneficial for solving stiff problems, as it can handle large time steps without leading to instability. For the HDE, the implicit method remains stable regardless of the chosen time step, making it unconditionally stable.
\begin{Proposition}[Implicit Euler Stability]\label{P8}
The implicit Euler method for HDE in Eq.(\ref{elueri}) is unconditionally stable. Specifically, for any step size \( \tau > 0 \), the solution will not grow unbounded.
\end{Proposition}
The proofs of Propositions \ref{P7} and Propositions \ref{P8} have given in Appendix \ref{app:proposition}.

\paragraph{Multi-step and Adaptive Integration Schemes.}
Our framework supports a range of integration schemes beyond Euler methods, including higher-order Runge--Kutta, linear multi-step methods (e.g., Adams--Bashforth and Adams--Moulton), and adaptive-step solvers. These methods offer different trade-offs between computational cost, stability, and accuracy, and can be naturally incorporated into our operator-based formulation. Detailed formulations are provided in Appendix~\ref{app:solvers}.

\paragraph{Relation to HGNNs}
\begin{itemize}
\item \textbf{Unconditional Stability for Deep Architectures:}
Implicit time-stepping schemes exhibit unconditional stability with respect to the step size 
$\tau$, ensuring that feature evolution remains bounded regardless of the temporal resolution. This property is particularly advantageous in HGNNs, where deeper architectures or high-cardinality hyperedges can exacerbate instability and lead to divergence in node representations. Implicit multi-step schemes offer a principled mechanism to structure layer-wise updates while providing strong stability guarantees, thereby facilitating the construction of deep and reliable HGNNs.

\item \textbf{Spectral Filtering and Eigenvalue Modulation:}
Both explicit and implicit integration schemes possess well-established spectral interpretations via the Rayleigh quotient and eigenvalue analysis. In the context of HGNNs, this translates to a control mechanism over how graph spectral components influence the diffusion dynamics. Specifically, implicit methods tend to attenuate high-frequency modes (associated with large eigenvalues of the hypergraph Laplacian), functioning effectively as low-pass filters. This spectral damping contributes to improved smoothness, robustness against noise, and enhanced generalization in the learned node embeddings.

\item \textbf{Adaptive Dynamics for Feature Sensitivity:}
Adaptive time-stepping strategies—such as those based on embedded RK methods—introduce a dynamic adjustment mechanism to the propagation process by estimating local truncation errors. When interpreted in the HGNN framework, such adaptation can be viewed as an implicit modulation of the aggregation intensity or learning rate at each node or hyperedge. This allows the network to adaptively respond to the local structure and signal variability, leading to a more expressive and structure-sensitive modeling of hypergraph data.

\item \textbf{Higher-Order Feature Propagation:}
Explicit multi-step methods, exemplified by the Adams–Bashforth scheme, incorporate temporal information from multiple previous states, effectively introducing higher-order dependencies into the propagation mechanism. In HGNNs, this corresponds to leveraging richer historical information across layers, enhancing the model’s capacity to capture long-range dependencies, preserve temporal consistency, and encode memory-like behavior in the learning dynamics.

\item \textbf{Discrete-Time Diffusion Interpretation of HGNNs:}
HGNN layer-wise updates can be rigorously interpreted as discrete-time approximations of nonlinear diffusion processes over hypergraphs. The choice of numerical integration scheme—explicit vs. implicit, fixed-step vs. adaptive—fundamentally determines the fidelity and stability of this approximation. This interpretation not only bridges HGNNs with classical PDE-based diffusion theory but also offers a unified analytical perspective for studying model dynamics, guiding architecture design, and informing the selection of training hyperparameters.
\end{itemize}

\section{Hypergraph Neural Diffusion}
Building on the continuous-time formulation of the HDE and its discretization through advanced numerical schemes, we propose a novel neural architecture termed Hypergraph Neural Diffusion (HND). This model unifies the principles of gradient-driven feature propagation, adaptive anisotropic diffusion, and neural parametrization within a coherent framework grounded in PDE theory.
\subsection{HND Framework.}
We now describe HND, a class of neural architectures derived from the nonlinear diffusion equation on hypergraphs in Eq.(\ref{E1}). Given a hypergraph $\mathcal{H}=(\mathcal{V},\mathcal{E},\mathcal{W})$ $n$ nodes and input features $\mathbf{X}_{in} \in \mathbb{R}^{n \times d_{\text{in}}}$, HND implements a learnable encoder–diffusion–decoder pipeline:
\begin{equation}
\mathbf{X}(0)=\phi(\mathbf{X}_{in}),\quad \mathbf{X}(T) = \mathbf{X}(0)+\int^T_0\frac{\partial \mathbf{X}(t)}{\partial t},\quad \mathbf{Y}=\psi(\mathbf{X}(T)),
\end{equation}
where $\phi$, $\psi$ are trainable encoder and decoder functions, and $\frac{\partial \mathbf{X}(t)}{\partial t}$ is governed by HND in Eq.(\ref{L1}). The hidden dynamics are parameterized via a nonlinear, learnable diffusivity operator $-G^\top \mathbf{A}_\theta(\mathbf{X}) G$, where $G$ is the hypergraph gradient matrix and  $\mathbf{A}_\theta$ is a learnable modulation matrix parameterized by a neural function.
\paragraph{Encoder: Initial Feature Projection.}
To project the raw features into the hidden space of the diffusion process, we define a learnable encoder $\phi:\mathbb{R}^{n \times d_{\text{in}}}\rightarrow\mathbb{R}^{n \times d}$ as:
\begin{equation}
    \mathbf{X}(0)=\phi(\mathbf{X}_{in})= \mathbf{X}_{in} \mathbf{W}_{in}
\end{equation}
where $\mathbf{W}_{in} \in \mathbb{R}^{d_{\text{in}} \times d}$ are trainable parameters projecting the input into the hidden feature space of dimension $d$. This transformation ensures compatibility across all diffusion layers and serves as a learnable encoder that adapts to the downstream task.
\paragraph{Parametrization of $\mathbf{A}_\theta(\mathbf{X})$.} The modulation matrix is diagonal and defined over the hyperedge–node pair space. For each pair $(e,v) \in \mathcal{I}$, we define:
\begin{equation}
\begin{aligned}
s_\theta(\mathbf{x}_e, \mathbf{x}_v) &= \sigma\left( \text{MLP}([\mathbf{x}_v \, \| \, \text{Agg}_u(\mathbf{x}_u : u \in e)]) \right),\\
a_\theta(\mathbf{x}_e, \mathbf{x}_v)&=
\frac{
\exp\big( s_\theta(\mathbf{x}_e, \mathbf{x}_v) \big)
}{
\sum_{e' \ni v} \exp\big( s_\theta(\mathbf{x}_{e'}, \mathbf{x}_v) \big)
},
\end{aligned}
\end{equation}
where $\|$ denotes concatenation, $\text{Agg}$ is a permutation-invariant aggregator (e.g., mean or max), and $\sigma$ is a $\operatorname{LeakyReLU}$ activation function. This parametrization allows the model to learn structure-aware diffusion strengths that vary across hyperedges and depend on local feature geometry. This formulation guarantees that $a_\theta(e,v) > 0$ and \(
\sum_{e \ni v} a_\theta(e,v) = 1,
\). This aligns with the theoretical assumptions made in Eq.(\ref{A1}) and guarantees the boundedness of feature dynamics across layers.

\paragraph{Neural Diffusion Layer.}
The evolution of features over diffusion layers is realized via discretizations of the underlying PDE. For instance, under a forward Euler scheme, the node feature matrix $\mathbf{X} \in \mathbb{R}^{n \times d}$ is updated via a structured propagation rule derived from:
\begin{equation}
\mathbf{X}^{(l+1)} = \mathbf{X}^{(l)} - \tau \cdot G^\top \mathbf{A}_\theta(\mathbf{X}^{(l)}) G \mathbf{X}^{(l)}=\mathbf{P}^{(l)}\mathbf{X}^{(l)},
\end{equation}
where $\mathbf{A}_\theta(\mathbf{X})=\operatorname{diag}(a_\theta(\mathbf{x}_e, \mathbf{x}_v))$ is a feature-adaptive modulation matrix, and each entry is predicted by a neural function (e.g., multilayer perceptron (MLP) or attention) based on local node–hyperedge information. This nonlinear operator $G^\top \mathbf{A}_\theta(\cdot) G$ acts as an adaptive hypergraph Laplacian encoding both topology and feature geometry.
\begin{algorithm}[ht]
\caption{Forward Propagation in HND Layer}
\label{alg:HND}
\textbf{Input:} Hypergraph $\mathcal{H}=(\mathcal{V}, \mathcal{E}, \mathcal{W})$ with incidence structure $\mathcal{I}$, feature matrix $\mathbf{X}_{\text{in}} \in \mathbb{R}^{n \times d_{\text{in}}}$, step size $\tau$, depth $L$, hidden dimension $d$. \\
\textbf{Output:} $\mathbf{X}^{(L)}\in\mathbb{R}^{n \times d_{\text{in}}}$.

\begin{algorithmic}[1]
\State $\mathbf{X}^{(0)} \leftarrow \phi(\mathbf{X}_{\text{in}}) = \mathbf{X}_{\text{in}} \mathbf{W}_{\text{in}}$
\For{$l = 0$ to $L-1$}
\Comment{Compute modulation weights for each hyperedge-node pair}
    \For{each $(e,v) \in \mathcal{I}$}
        \State $\mathbf{x}_e^{(l)} \leftarrow \text{Agg}_u(\mathbf{x}^{(l)}_u : u \in e)$
        \State $s_\theta(\mathbf{x}_e, \mathbf{x}_v) \leftarrow \sigma\left(\text{MLP}([\mathbf{W}\mathbf{x}^{(l)}_v \, \| \, \mathbf{W}\mathbf{x}_e^{(l)}])\right)$
    \EndFor
    \For{ each $v \in \mathcal{V}$} \Comment{Normalize modulation weights for each node}
           \State $a_\theta(\mathbf{x}_e, \mathbf{x}_v) \leftarrow s_\theta(\mathbf{x}_e, \mathbf{x}_v) / \sum_{e' \ni v} s_\theta(\mathbf{x}_{e'}, \mathbf{x}_v)$
\EndFor
   \State Construct $\mathbf{A}_\theta^{(l)} = \operatorname{diag}(a_\theta(\mathbf{x}_e, \mathbf{x}_v))$
    \Comment{Assemble diagonal modulation matrix $\mathbf{A}_\theta^{(l)}$}
   \State $\mathbf{X}^{(l+1)} \leftarrow \mathbf{X}^{(l)} - \tau \cdot G^\top \mathbf{A}_\theta^{(l)} G \mathbf{X}^{(l)}$
   \Comment{Perform diffusion update}
\EndFor
\end{algorithmic}
\end{algorithm}

\paragraph{Computational Complexity.}
Let \( n = |\mathcal{V}| \) be the number of nodes, \( m = |\mathcal{E}| \) the number of hyperedges, and let each hyperedge have on average \( c \) nodes, so that the number of hyperedge–node pairs satisfies \( N = |\mathcal{I}| \approx cm \). Assume hidden feature dimension \( d \).

In each layer of Algorithm~\ref{alg:HND}, the main computational components are as follows:
\begin{itemize}
\item The core diffusion step involves computing \( G\mathbf{X} \in \mathbb{R}^{cm \times d} \), modulation via \( \mathbf{A}_\theta(\mathbf{X}) \), and backward aggregation \( G^\top \mathbf{A}_\theta G \mathbf{X} \in \mathbb{R}^{n \times d} \). This gives a per-layer cost of \( \mathcal{O}(cmd) \).
\item The modulation matrix \( \mathbf{A}_\theta \) is computed using a neural function per hyperedge–node pair, leading to a cost of \( \mathcal{O}(cmd) \).
\item Encoder and decoder have complexity \( \mathcal{O}(ndd_{\text{in}} + nd_{\text{out}}d) \), typically negligible compared to diffusion for large \( m \).
\end{itemize}

Hence, the overall per-layer complexity is:
\[
\mathcal{O}(cmd) \quad \text{(dominant term)}.
\]

This complexity is linear in the number of hyperedges and their average size, making HND efficient for sparse hypergraphs with small \( c \).

\subsection{Feature-Adaptive Control of Diffusion Strength}
In classical models based on fixed (e.g., normalized) Laplacians, the diffusion operator applies a uniform smoothing effect across all nodes and edges, inevitably leading to a collapse of expressiveness in deeper architectures.

In contrast, the feature-adaptive modulation matrix $\mathbf{A}_\theta(\mathbf{X})$ in HND introduces learnable heterogeneity into the diffusion process. This operator is parameterized by neural networks and varies across both nodes and hyperedges based on local feature geometry. Formally, the diffusion operator becomes:
\[
 \mathcal{L}_{\theta}(\mathbf{X}) = G^\top \mathbf{A}_\theta(\mathbf{X}) G \mathbf{X} 
\]
which generalizes the classical Laplacian to a data-dependent form. Unlike classical Laplacians, this operator varies with the input features and enables adaptive control of the diffusion process.

Specifically, the modulation matrix $\mathbf{A}_\theta(\mathbf{X})$ allows the model to adjust diffusion strength across different hyperedge–node pairs based on local feature geometry. As a result, smoothing is no longer applied uniformly, but can be selectively attenuated in regions where feature differences are significant. This provides a mechanism to balance smoothness and feature discrimination. Our formulation operates at the level of hyperedge–node interactions and is embedded within a diffusion operator derived from variational principles. This structural difference enables a more direct control of the underlying diffusion dynamics. We emphasize that our framework does not eliminate over-smoothing in a strict theoretical sense, but provides a flexible mechanism to regulate the strength and locality of diffusion through learnable modulation.

\subsection{Layer Variants via Numerical Schemes}

The HND framework supports a broad family of layer-wise update rules, each corresponding to a different time discretization of the continuous diffusion process. This versatility enables tailored architectural designs based on task-specific accuracy, stability, and computational constraints. Below, we elaborate on several key variants.

\paragraph{Explicit Euler HND.}
This variant applies a straightforward forward Euler scheme:
\[
\mathbf{X}^{(l+1)} = \mathbf{X}^{(l)} - \tau \mathcal{L}_{\theta}(\mathbf{X}^{(l)}) \mathbf{X}^{(l)}.
\]
It is computationally efficient, requiring only one matrix-vector product per layer. However, to guarantee numerical stability and avoid feature explosion, the step size \( \tau \) must be carefully selected—typically constrained by the spectral norm of the diffusion operator.

\paragraph{Implicit Euler HND.}
This variant replaces the explicit update with an implicit one:
\[
\mathbf{X}^{(l+1)} = \mathbf{X}^{(l)} - \tau \mathcal{L}_{\theta}(\mathbf{X}^{(l+1)}) \mathbf{X}^{(l+1)}.
\]
This results in a nonlinear system that must be solved at each layer. In practice, iterative solvers such as fixed-point iteration or gradient descent are used to approximate \( \mathbf{X}^{(l+1)} \). Although more computationally demanding, this scheme offers superior stability and robustness, especially beneficial in deep networks or for stiff diffusion operators.

\paragraph{Advanced Integration Schemes.}
In addition to Euler-based discretizations, our framework supports a range of advanced schemes, including adaptive-step methods, higher-order Runge--Kutta, and linear multi-step methods. These approaches enable more flexible control of diffusion dynamics, offering different trade-offs between computational cost, stability, and accuracy. All these schemes can be naturally integrated into our operator-based formulation, yielding corresponding hypergraph neural architectures. Detailed formulations are provided in Appendix~\ref{app:models}.

The above integration schemes highlight different trade-offs between computational cost, numerical stability, and approximation accuracy. Let $|\mathcal{I}|$ denote the number of hyperedge–node pairs and $d$ the feature dimension. A single evaluation of the diffusion operator $\mathcal{L}_\theta(\mathbf{X})$ requires $O(|\mathcal{I}|\, d)$ operations.

In particular, implicit methods improve stability by allowing larger step sizes, but require solving a nonlinear system at each layer, typically via $K$ (e.g., 3–10) iterations, leading to a complexity of $O(K\,|\mathcal{I}|\, d)$ per layer. Higher-order methods such as Runge--Kutta enhance accuracy through multiple operator evaluations, resulting in a complexity of $O(s\,|\mathcal{I}|\, d)$, where $s$ is the number of stages (e.g., $s=4$ for RK4).

These trade-offs are particularly relevant for nonlinear and potentially stiff diffusion dynamics, where the choice of numerical scheme directly affects both efficiency and stability. Our framework provides a unified formulation that supports these schemes, allowing practitioners to choose appropriate discretizations based on computational budget and accuracy requirements.

\subsection{Theoretical Advantages and Interpretability}

The HND model inherits several theoretical guarantees from its continuous diffusion roots:
\begin{itemize}
\item \textbf{Energy-Based Stability and Adaptive Dynamics.}  
The HND framework enjoys strong numerical and theoretical guarantees rooted in the variational structure of the nonlinear HDE. By Proposition~\ref{P4}, each diffusion layer performs a descent step on the energy functional \(\mathcal{E}(\mathbf{x})\), ensuring that \(\mathcal{E}(\mathbf{X}^{(l+1)}) \leq \mathcal{E}(\mathbf{X}^{(l)})\). This leads to progressive reduction of intra-hyperedge inconsistency and contributes to convergence and representational stability. 

In terms of numerical discretization, explicit Euler-based HND layers are provably stable under the condition \(\tau \leq 1\) (Proposition~\ref{P7}), while implicit Euler variants are unconditionally stable for any step size \(\tau > 0\) (Proposition~\ref{P8}). These discretization choices allow practitioners to trade off computational efficiency and robustness, especially when scaling to deeper architectures or stiffer diffusion regimes.

Moreover, HND supports higher-order and adaptive integration schemes, such as RK and Adams–Bashforth/Moulton, which offer improved approximation accuracy and support memory-aware propagation across layers. Embedded error estimators further allow dynamic adjustment of \(\tau\) in response to signal variability or local complexity, enabling structure-sensitive and stable information flow. Together, these properties form a principled and flexible backbone for layer-wise neural computation grounded in PDE theory.

\item \textbf{Structure-Aware and Adaptive Diffusion.}  
The learned modulation matrix \(\mathbf{A}_\theta(\mathbf{X})\), parameterized by MLPs or attention mechanisms and subject to the structural normalization constraint in Eq.(\ref{A1}), enables feature-adaptive, anisotropic control over diffusion strength. This structure-sensitive design empowers the model to modulate information flow based on local node–hyperedge relationships, allowing nuanced responses to complex graph structures and semantic patterns.

In particular, this mechanism enables non-uniform smoothing across the hypergraph. By assigning low modulation weights to semantically dissimilar node–hyperedge pairs, \(\mathbf{A}_\theta(\mathbf{X})\) suppresses excessive mixing that can obscure meaningful signal variations. Conversely, it amplifies smoothing in coherent regions, promoting discriminative yet stable feature propagation. This heterogeneity in the effective Laplacian \(G^\top \mathbf{A}_\theta(\mathbf{X}) G\) delays representation collapse and maintains the expressivity of node embeddings across multiple diffusion layers.

\item \textbf{Boundedness via Maximum Principle.}  
As shown in Proposition~\ref{P6}, the nonlinear HDE satisfies a discrete maximum principle, ensuring that each node's feature values remain within the convex hull of its initial range. In the HND architecture, this guarantees that:
\[
\min_{v \in \mathcal{V}} \frac{\mathbf{x}_v^{(0)}}{\sqrt{d_v}} \leq \frac{\mathbf{x}_v^{(l)}}{\sqrt{d_v}} \leq \max_{v \in \mathcal{V}} \frac{\mathbf{x}_v^{(0)}}{\sqrt{d_v}},\quad \forall l \geq 0,v\in\mathcal{V},
\]
which protects against feature explosion or vanishing and improves interpretability by preserving the semantic range of node representations.

\item \textbf{Unified PDE-Inspired Perspective.}  
Finally, HND offers a principled connection between numerical PDE solvers and HGNN design. Each layer can be interpreted as a step in a discretized nonlinear diffusion process, governed by a data-dependent Laplacian operator. This connection bridges continuous-time dynamics with neural message passing and provides theoretical tools for analyzing and improving HGNNs.

\end{itemize}

\section{Experiments}
To comprehensively evaluate the performance of the proposed framework, we conducted a series of systematic experimental studies. Specifically, we first present extensive performance evaluations on multiple real-world and synthetic datasets, demonstrating the effectiveness of the model. Secondly, we provide an in-depth analysis of the model's behavior in terms of over-smoothing and robustness. Finally, we investigate the model's sensitivity to hyperparameters and feature visualization, offering valuable insights into its interpretability and adaptability in practical applications. Source code is available at \href{https://gitee.com/zmy-ovo/hnd}{https://gitee.com/zmy-ovo/hnd}.
\subsection{Results on Benchmarking Datasets}
\textbf{Datasets}. To comprehensively evaluate the performance of the proposed method, we conducted systematic experiments on multiple benchmark datasets, including both academic and real-world scenarios. The academic datasets consist of five widely used hypergraph benchmarks~\citep{yadati2019hypergcn}: Cora, Citeseer, Pubmed (co-citation networks), and Cora-CA, DBLP-CA (co-authorship networks), where node features are bag-of-words representations and labels denote paper categories. The real-world datasets include 20Newsgroups and Zoo from the UCI repository~\citep{dua2017uci}, ModelNet40~\citep{wu20153d} and NTU2012~\citep{chen2003visual} from the 3D vision domain, as well as the House~\citep{chodrow2021generative} and Senate datasets ~\citep{fowler2006connecting} from social and political networks. Hypergraph structures are constructed following prior works, and for datasets lacking node features, Gaussian random vectors are used. All datasets are split into training, validation, and test sets with a ratio of 50\%, 25\%, and 25\%. For evaluation, we generate 20 different train/validation/test splits using a deterministic pseudo-random sequence initialized with a fixed base seed , and report the aggregated results over these 20 runs.
The partial statistics of all
datasets are provided in Table \ref{tab:dataset_stats} .

\begin{table}[htbp]
  \centering
  \caption{Dataset Statistics Summary\label{tab:dataset_stats}}
 \tiny
  \setlength{\tabcolsep}{4pt} 
  \begin{tabular}{l *{12}{S[table-format=5.0]}} 
    \toprule
    {Metric} & {Cora} & {Citeseer} & {Pubmed} & {Cora-CA} & {DBLP-CA} & {Zoo} & {20News}  & {NTU2012} & {ModelNet40} & {House} & {Senate}  \\
    \midrule
    $|V|$ & 2708 & 3312 & 19177 & 2708 & 41302 & 101 & 16242 & 2012 & 12311 & 1290& 282  \\
    $|E|$ & 1579 & 1079 & 7963 & 1072 & 22363 & 43 & 100 &  2012 & 12311 & 340 & 315 \\
    feat& 1433 & 3703 & 500 & 1433 & 1425 & 16 & 100 & 100 & 100 & 100 & 100  \\
    class & 7 & 6 & 3 & 7 & 6 & 7 & 4 &  67 & 40 & 2 & 2\\
    \bottomrule
  \end{tabular}
  \vspace{0.2cm}
\end{table}

\paragraph{Baselines.}
We compare our model with various hypergraph neural network baselines. These include HGNN~\citep{feng2019hypergraph}, which introduces spectral convolution operations tailored for hypergraph-structured data; HCHA~\citep{bai2021hypergraph}, which incorporates hierarchical attention mechanisms for hypergraph representation learning; and HyperGCN ~\citep{yadati2019hypergcn}, which extends graph convolutional networks to hypergraphs via clique expansion. HNHN ~\citep{dong2020hnhn} leverages novel normalization techniques specific to hypergraphs, while UniGCNII ~\citep{huang2021unignn} unifies various hypergraph convolution paradigms with residual connections. HAN ~\citep{wang2019heterogeneous} is also included for its use of hierarchical attention networks in hypergraph analysis.

Additionally, we evaluate against AllSetTransformer and AllDeepSets~\citep{chien2022you}, which adapt deep set operations, respectively, to capture permutation-invariant properties in hypergraph learning. We also include ED-HNN ~\citep{wang2023equivariant}, which introduces equivariant  hypergraph diffusion
operators, as well as HyperGINE and KHGNN~\citep{xie2025k}, where KHGNN enables effective interaction between distant nodes and hyperedges through K-hop message passing. In addition, we consider FrameHGNN~\citep{li2025deep}, which integrates framelet transforms with graph neural networks to capture multi-scale hypergraph representations. We further include HNSD~\citep{choi2025hypergraph}, which models hypergraph diffusion via neural spectral decomposition to capture higher-order structural patterns. HyperUFG~\citep{li2025hypergraph} is also considered, which leverages unified framelet-based filtering to enable flexible frequency-domain representations on hypergraphs. Finally, we include HyperSheaflets~\citep{li2026high}, which introduces sheaf-theoretic constructions to model local geometric consistency and high-order interactions in hypergraph data. All models are implemented using the PyTorch Geometric framework ~\citep{fey2019fast} to ensure a fair and consistent comparison.
\paragraph{Experiment Setting.} We study two variants of HND: linear and nonlinear. In the HND-L, the attention weights are constant throughout the integration, producing a coupled system of linear ODEs. In HND-NL, the attention weights are updated at each step of the numerical integration. In both cases, the given hypergraph is used as the spatial discretisation of the diffusion operator.
For both HND-L and HND-NL models, we use the Adam optimizer with a fixed learning rate of 0.01. The key hyperparameters—weight decay, input dropout rate, hidden dimension, and total training time are selected individually for each dataset based on validation performance. Specifically, weight decay is chosen from $\left\{0.001, 0.01, 0.1\right\}$, input dropout from $\left\{0.001, 0.01, 0.1, 0.2, 0.3\right\}$, hidden dimension from $\left\{16, 32, 64, 128, 256,512\right\}$, and training time from $\left\{4, 5, 6,7,8\right\}$. For both models, the best configurations are selected  within these ranges. The learning rate scheduler is set to CosineLR when applicable. All experiments use a fixed base random seed (seed = 0) for reproducibility. The reported results are averaged over 20 distinct random train/validation/test splits.

\paragraph{Results.} As shown in Table \ref{tab2: academic datasets} and Table \ref{tab3: real-world datasets}, experimental results demonstrate that the proposed method nearly outperforms baseline models across a wide range of datasets, including five academic and six real-world datasets. Specifically, our model achieves state-of-the-art performance on two out of the five academic benchmarks, showcasing its ability to effectively capture complex relationships in structured data. Furthermore, it nearly outperforms all baseline methods on the six real-world datasets, which span diverse application domains such as 3D object recognition and social network analysis. These results collectively confirm the effectiveness of HND in both controlled academic settings and practical real-world scenarios. Additionally, we computed the average ranks of the methods, with HND-L and HND-NL consistently securing the top two ranks across both academic and real-world datasets, further underscoring their superior performance.

\begin{table}[htbp!]
\centering
\scriptsize
\caption{Performance comparison on academic hypergraph datasets (Mean accuracy (\%) ± standard deviation), with the best results in \textcolor{red}{red}, second-best in \textcolor{orange}{orange}, and third-best in \textcolor{OliveGreen}{green}.}
\label{tab2: academic datasets}
\begin{tabular}{ccccccc}
\hline
 Models& Cora & Citeseer & Pubmed&Cora-CA&DBLP-CA&Rank$\downarrow$ \\
\hline
HNHN & 76.36$\pm$1.92 & 72.64$\pm$1.57 & 86.90$\pm$0.30 & 77.19$\pm$1.49 & 86.78$\pm$0.29& 15\\
HGNN & 79.39$\pm$1.36 & 72.45$\pm$1.16 & 86.44$\pm$0.44 & 82.64$\pm$1.65 & 91.03$\pm$0.20 &12 \\
 HCHA& 79.14$\pm$1.02 & 72.42$\pm$1.42 & 86.41$\pm$0.36 & 82.55$\pm$0.97 & 90.92$\pm$0.22 & 14\\
HyperGCN & 78.45$\pm$1.26 & 71.28$\pm$0.82 & 82.84$\pm$8.67 & 79.48$\pm$2.08 & 89.38$\pm$0.25&16 \\
UniGCNII& 78.81$\pm$1.05 & 73.05$\pm$2.21 & 88.25$\pm$0.40 & 83.60$\pm$1.14 & \textcolor{OliveGreen}{\textbf{91.69$\pm$0.19}} &9 \\
AllSetTransformer& 78.59$\pm$1.47 & 73.08$\pm$1.20 & 88.72$\pm$0.37 & 83.63$\pm$1.47 & 91.53$\pm$0.23  &8\\
AllDeepSets & 76.88$\pm$1.80 & 70.83$\pm$1.63 & \textcolor{orange}{\textbf{88.75$\pm$0.33}}& 81.97$\pm$1.50 & 91.27$\pm$0.27 &13 \\
HAN & 79.70$\pm$1.77 & 74.12$\pm$1.52 & 85.32$\pm$2.25 & 81.71$\pm$1.73 & 90.17$\pm$0.65  & 13 \\
ED-HNN & 80.31$\pm$1.35 & 73.70$\pm$1.38 & 
\textcolor{red}{\textbf{89.03$\pm$0.53}} & 83.97$\pm$1.55 & 
\textcolor{red}{\textbf{91.90$\pm$0.19}} & 6 \\
HyperGINE & 79.26$\pm$0.41 & 73.72$\pm$0.52 & 87.91$\pm$0.28 & 82.88$\pm$0.48 & - &10\\
KHGNN& 80.67$\pm$0.76 & 74.80$\pm$1.10 & 88.47$\pm$0.47 & 84.25$\pm$0.74 &-& 7 \\
FrameHGNN & {81.51$\pm$0.99} & 74.72$\pm$2.10 &\textcolor{OliveGreen}{\textbf{88.73$\pm$0.42}} & 85.18$\pm$0.69 &-& 4 \\
 HNSD & {{79.28$\pm$0.82}} & 74.40$\pm$1.47 &{{-}} & {{82.58$\pm$1.15}} &{{89.85$\pm$0.44}}& 11\\
HyperUFG & {{81.51$\pm$0.99}} & {{74.72$\pm$2.10}} & \textcolor{OliveGreen}{\textbf{88.73$\pm$0.42}} & 85.18$\pm$0.69
& {{91.67$\pm$0.31}} & \textcolor{OliveGreen}{\textbf{3}}\\
HyperSheaflets & \textcolor{OliveGreen}{\textbf{81.60$\pm$1.92}}  & \textcolor{OliveGreen}{\textbf{75.19$\pm$1.80}} & 87.19$\pm$0.45 & \textcolor{red}{\textbf{85.85$\pm$0.92}} & 91.58$\pm$0.27 & 5\\
\hline
HND-L &
\textcolor{red}{\textbf{81.76$\pm$1.12}} & \textcolor{orange}{\textbf{75.51$\pm$1.37}} & 88.52$\pm$0.30 &
\textcolor{orange}{\textbf{85.49$\pm$1.05}} & \textcolor{orange}{\textbf{91.71$\pm$0.25}}&\textcolor{red}{\textbf{1}}\\
HND-NL &
\textcolor{orange}{\textbf{81.63$\pm$1.24}}  & 
\textcolor{red}{\textbf{75.80$\pm$1.28}} &88.56$\pm$0.23 &\textcolor{OliveGreen}{\textbf{85.44$\pm$1.77}}&91.29$\pm$0.23 &\textcolor{orange}{\textbf{2}}\\
\hline
\end{tabular}
\end{table}

\begin{table}[htbp!]
\tiny
\centering
\caption{Performance comparison on real-world hypergraph datasets (Mean accuracy (\%) ± standard deviation), with the best results in \textcolor{red}{red}, second-best in \textcolor{orange}{orange}, and third-best in \textcolor{OliveGreen}{green}.}
\label{tab3: real-world datasets}
\begin{tabular}{cccccccc}
\hline
 Models&  Zoo&20Newsgroup& NTU2012 & ModelNet40&Senate&House&Rank$\downarrow$ \\
\hline
HNHN & 93.59$\pm$5.88 &81.35$\pm$0.61& 89.11$\pm$1.44&97.84$\pm$1.25& 50.93$\pm$6.33&67.80$\pm$2.59 &7\\
HGNN & 92.50$\pm$4.58 &80.33$\pm$0.42& 87.72$\pm$1.35 & 95.44$\pm$0.33 & 48.59$\pm$4.52&61.39$\pm$2.96&12\\
HCHA & 93.65$\pm$6.15 &80.33$\pm$0.80& 87.48$\pm$1.87& 94.48$\pm$0.28 & 48.62$\pm$4.41&61.36$\pm$2.53&11 \\
HyperGCN & N/A &81.05$\pm$0.59& 56.36$\pm$4.86 & 75.89$\pm$5.26 & 42.45$\pm$3.67&48.32$\pm$2.93&14 \\
UniGCNII & 93.65$\pm$4.37 &81.12$\pm$0.67& 89.30$\pm$1.33 & 98.07$\pm$0.23 & 49.30$\pm$4.25&61.70$\pm$3.37&8 \\
AllSetTransformer & \textcolor{OliveGreen}{\textbf{97.50$\pm$3.59}}&\textcolor{OliveGreen}{\textbf{81.38$\pm$0.58}}& 88.69$\pm$1.24&98.20$\pm$0.20& 51.83$\pm$5.22 &69.33$\pm$2.20& 6 \\
AllDeepSets & 95.39$\pm$4.77 &81.06$\pm$0.54& 88.09$\pm$1.52&96.98$\pm$0.26& 48.17$\pm$5.67&67.82$\pm$2.40&10 \\
HAN & 75.77$\pm$7.10 & 79.72$\pm$0.62& 83.58$\pm$1.46 & 94.04$\pm$0.41 &  & 62.00$\pm$9.09&13 \\
ED-HNN & - & - & 88.67$\pm$0.92 & 97.83$\pm$0.33 & 64.79$\pm$5.14&72.45$\pm$2.28&8 \\
HyperGINE & - & - & 88.52$\pm$0.42 & 97.61$\pm$0.16 & - & - &10\\
KHGNN & - & - & 89.60$\pm$1.64 & 98.33$\pm$0.14 & - & -& 4 \\
FrameHGNN  &-  & - &\textcolor{OliveGreen}{\textbf{89.98$\pm$2.02}}& \textcolor{OliveGreen}{\textbf{98.41$\pm$0.18}} & 67.61$\pm$5.27& 72.82$\pm$2.22 & 4 \\
 HNSD & {{90.20$\pm$5.87}} & - &{{88.31$\pm$1.67}} & {{97.42$\pm$0.61}} &\textcolor{red}{\textbf{78.45$\pm$5.87}}&{{61.04$\pm$2.61}}& 9 \\
HyperUFG &  & & & & 67.61$\pm$7.00 & 72.82$\pm$2.22 & 5\\
HyperSheaflets &  & & & & 69.01$\pm$5.39 & \textcolor{orange}{\textbf{74.49$\pm$1.21}} & \textcolor{OliveGreen}{\textbf{3}}\\
\hline
HND-L & 
\textcolor{red}{\textbf{99.19$\pm$1.24}} & \textcolor{orange}{\textbf{82.40$\pm$1.40}} & 
\textcolor{red}{\textbf{93.32$\pm$0.99}} & \textcolor{orange}{\textbf{98.48$\pm$0.20} }& \textcolor{OliveGreen}{\textbf{70.00$\pm$4.44}} & \textcolor{OliveGreen}{\textbf{73.69$\pm$2.30}}&\textcolor{orange}{\textbf{2}} \\
HND-NL & \textcolor{orange}{\textbf{98.24$\pm$1.77}} & 
\textcolor{red}{\textbf{82.72$\pm$0.68}} &  \textcolor{orange}{\textbf{92.15$\pm$0.81}}&\textcolor{red}{\textbf{98.49$\pm$0.20}} & 
\textcolor{orange}{\textbf{70.97$\pm$4.24}} & 
\textcolor{red}{\textbf{74.63$\pm$2.11}}&\textcolor{red}{\textbf{1}} \\
\hline
\end{tabular}
\end{table}
\subsection{Results on Synthetic Heterophilic Hypergraph Dataset}
\paragraph{Experiment Setting.}
To demonstrate the effectiveness of HND across datasets with different levels of structural heterogeneity, we evaluate its performance using synthetic datasets with controlled levels of heterophily. These datasets are generated based on the contextual hypergraph stochastic block model~\citep{deshpande2018contextual, ghoshdastidar2014consistency, chien2018community}. Specifically, we construct two classes with 2,500 nodes each, and randomly sample 1,000 hyperedges. Each hyperedge contains 15 nodes, with $\alpha_i$ nodes sampled from class $i$. The heterophily level is defined as $\alpha = \min\{\alpha_1, \alpha_2\}$.

Node features are generated from label-dependent Gaussian distributions with a standard deviation of 1.0. We evaluate both homophilic settings ($\alpha = 1, 2$ or CE homophily $\geq 0.7$) and heterophilic settings ($\alpha = 4 \sim 7$ or CE homophily $\leq 0.7$). Our method, HND, is compared with several models, including HGNN, HyperGCN, and ED-HNN.  Following standard practice, we adopt a 50\%/25\%/25\% split for training, validation, and testing, respectively. To ensure robustness, we conduct experiments using 10 different random data splits and report the aggregated results.
\paragraph{Results.}
The experimental results in Table \ref{tab4: SY datasets} demonstrate that the HND model consistently outperforms other methods across all datasets. Notably, its superiority becomes more pronounced in heterophilic regions (when $\alpha$ \textgreater 3), where the model exhibits significantly enhanced robustness and generalization capability, thereby validating the effectiveness of its architectural design.

\begin{table}[htbp]
\tiny
\centering
\caption{Model Performance On Synthetic Hypergraphs with Controlled Heterophily $\alpha$.}
\label{tab4: SY datasets}
\begin{tabular}{ccccccc}
\toprule
\multirow{2}{*}{Model} & \multicolumn{3}{c}{Homophily} & \multicolumn{3}{c}{Heterophily} \\
\cmidrule(lr){2-4} \cmidrule(lr){5-7}
& {$\alpha = 1$} & {$\alpha = 2$} & {$\alpha = 3$} & {$\alpha = 4$} & {$\alpha = 6$} & {$\alpha = 7$} \\
\midrule
HGNN & 92.58 $\pm$ 0.57 & 87.40 $\pm$ 0.89 & 82.32 $\pm$ 0.84 & 77.18 $\pm$ 1.08 & 69.98 $\pm$ 0.85 & 68.41 $\pm$ 1.39 \\
HyperGCN  & 83.40 $\pm$ 1.67 & 78.13 $\pm$ 1.57 & 77.70 $\pm$ 1.01 & 74.90 $\pm$ 1.31 & 52.11 $\pm$ 0.75 & 49.40 $\pm$ 1.63 \\
UniGCN    & 90.42 $\pm$ 1.04 & 82.12 $\pm$ 0.67 & 80.04 $\pm$ 0.98 & 76.86 $\pm$ 1.10 & 52.30 $\pm$ 1.18 & 49.20 $\pm$ 0.74 \\
EDGNN     & 93.88 $\pm$ 0.67 & 89.89 $\pm$ 0.59 & 83.52 $\pm$ 0.85 & 76.65 $\pm$ 1.21 & 52.98 $\pm$ 0.78 & 49.65 $\pm$ 1.16 \\
HND-NL   & 95.02 $\pm$ 0.55 & 90.34 $\pm$ 0.85 & \textbf{85.20 $\pm$ 1.17} & 79.41 $\pm$ 0.90 & 73.47 $\pm$ 0.95 & 73.02 $\pm$ 0.93 \\
HND-L    & \textbf{95.34 $\pm$ 0.44 }& \textbf{90.50 $\pm$ 0.68} &84.33 $\pm$ 0.90  & \textbf{80.32 $\pm$ 0.92} & \textbf{73.74 $\pm$ 0.94 }& \textbf{73.28 $\pm$ 0.96} \\
\bottomrule
\end{tabular}
\end{table}

\subsection{Depth Sensitivity Analysis}
In this section, we analyze the impact of varying the number of layers on model performance. As noted in~\citep{chen2022preventing}, most existing hypergraph neural network architectures are shallow, which limits their capacity to capture information from higher-order neighbors. However, increasing the number of layers can also lead to performance degradation, a phenomenon commonly observed in graph and hypergraph models. To explore this phenomenon, we evaluate the effect of model depth on the Cora dataset. Specifically, under a fixed data split, we assess the performance of different models with layer configurations of 2, 4, 10, 20, 30, and 40.
\begin{figure}[b]
    \centering
    \includegraphics[width=10cm,height =5cm]
    {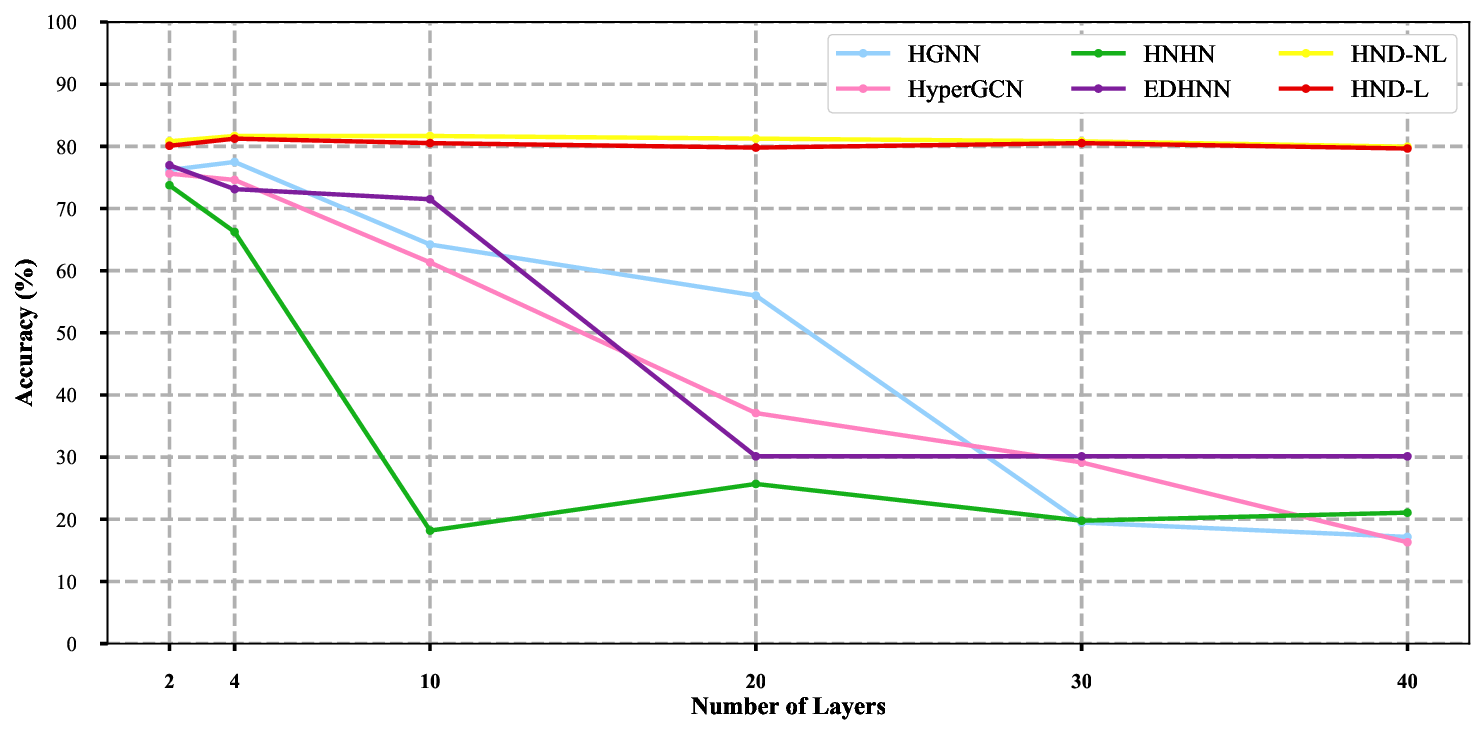}  
    \caption{Accuracy(\%)  with various
 depths on Cora.}  
    \label{fig:oversmoothing}    
\end{figure}

As illustrated in Figure~\ref{fig:oversmoothing}, both HND-NL (red) and HND-L (yellow) consistently achieve superior performance across different layer depths. Notably, their performance remains stable as the number of layers increases. In contrast, other methods tend to perform better with fewer layers but exhibit progressively degraded performance as the network becomes deeper. This demonstrates that HND effectively preserves the heterogeneity of node representations in deep layers, thereby supporting our theoretical analysis regarding its hierarchical feature preservation capability.
\begin{figure}[t]
\centering  
\subfigure[Gaussian Noise]
{\label{fig:gaussian}
\includegraphics[width=7cm,height =5cm]
{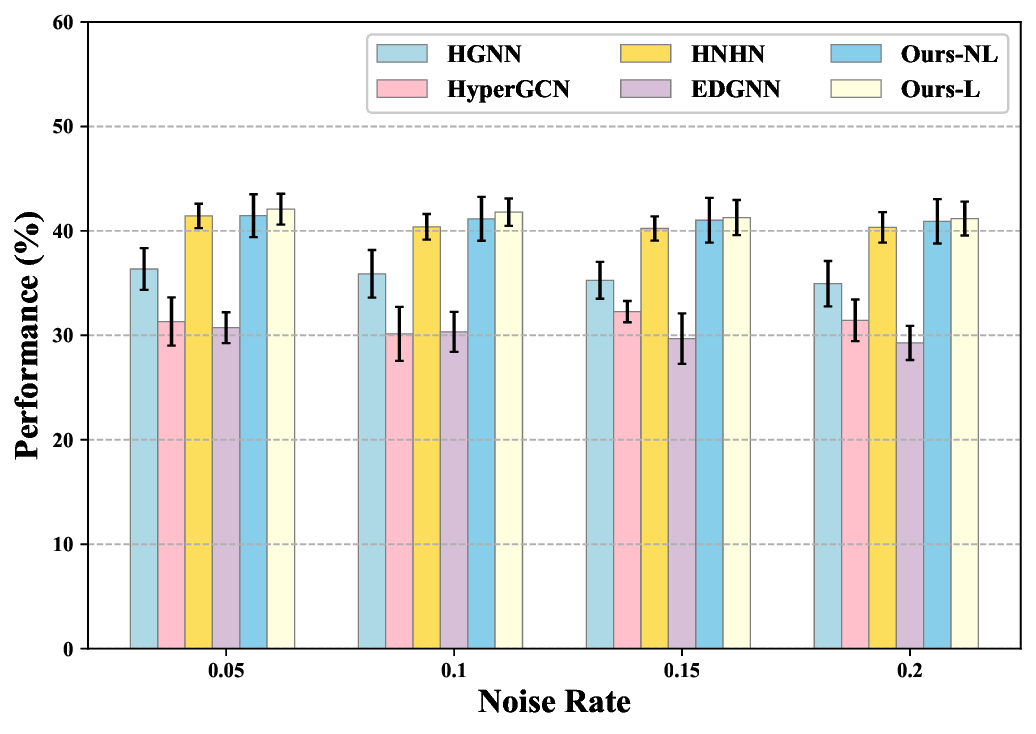}}
\subfigure[Uniform Noise]
{\label{fig:uniform}
\includegraphics[width=7cm,height =5cm]{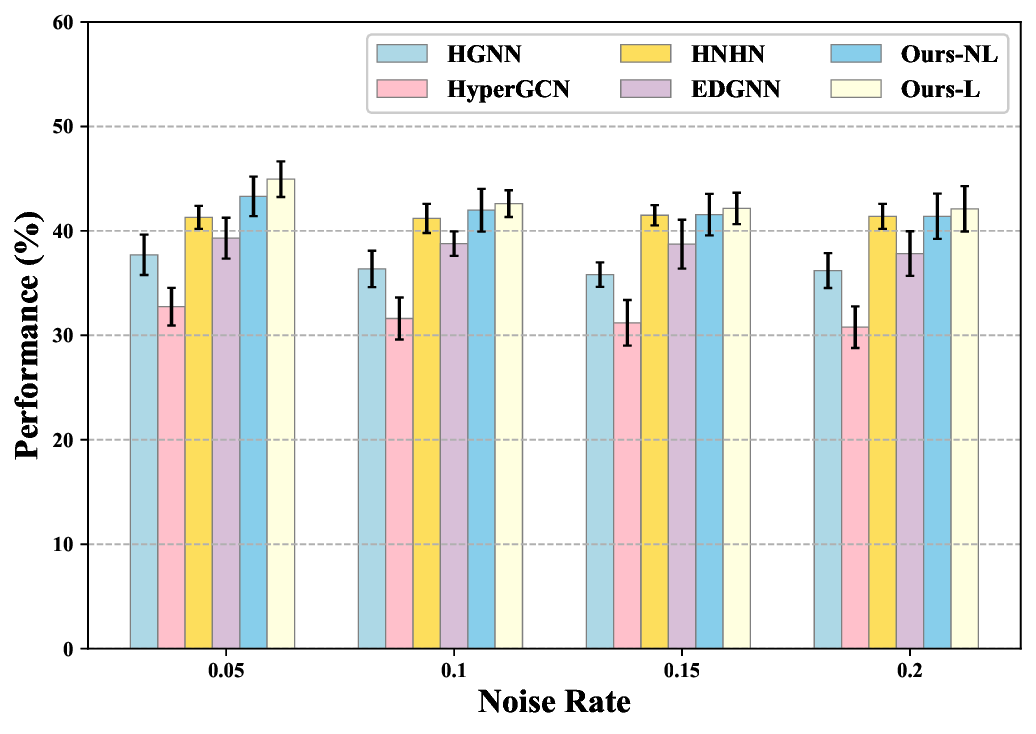}}
\subfigure[Mask Noise]
{\label{fig:mask}
\includegraphics[width=7cm,height =5cm]{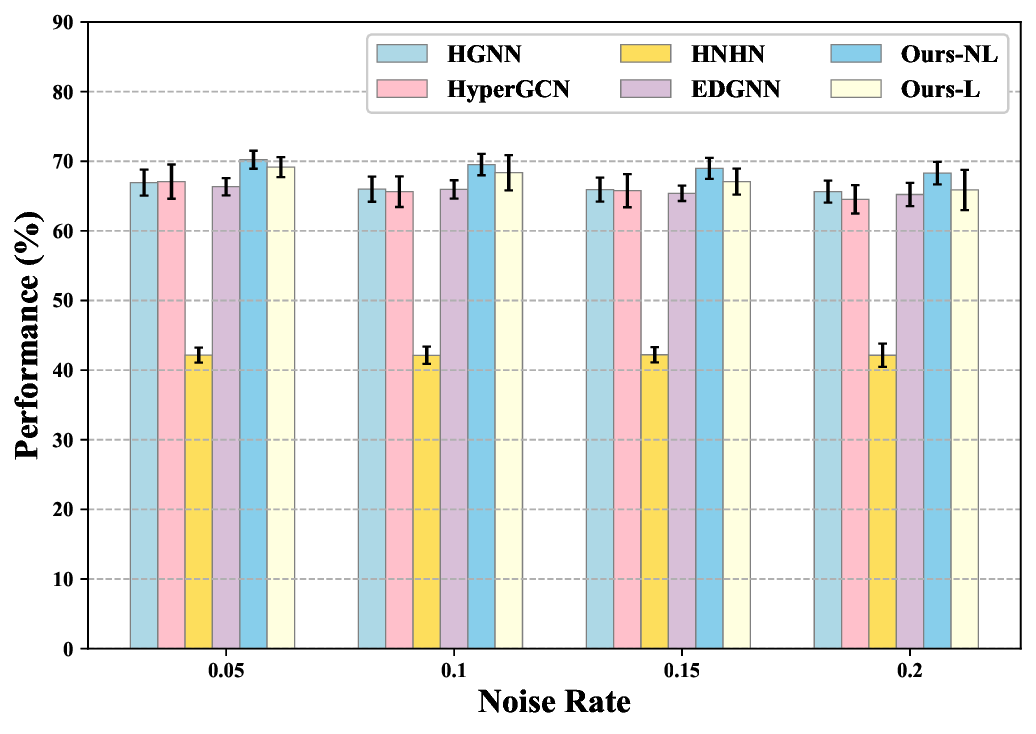}}
\subfigure[Structure Noise]
{\label{fig:stru}
\includegraphics[width=7cm,height =5cm]{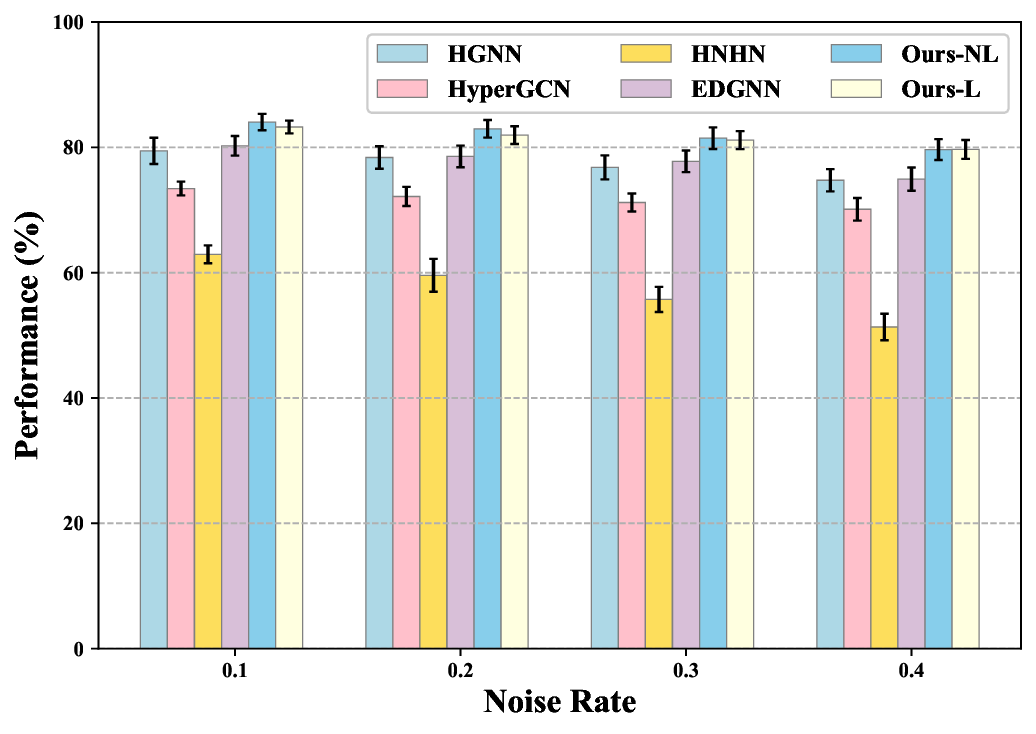}}
\caption{(a) shows the performance on Citeseer under feature  gaussian noise. (b) shows the performance on Citeseer under feature  uniform noise. (c) shows the performance on Citeseer under feature mask noise. (d) shows the performance on Cora-CA under structure noise.}
\label{noise}
\end{figure}

\subsection{Robustness analysis}
To evaluate the robustness of HND against noisy inputs, we conduct comprehensive experiments under both feature-level and structure-level perturbations. The experiments are designed as follows:
\begin{itemize}
 \item Feature Noise:
We corrupt node features on the Citeseer dataset with three types of noise:
(a) Gaussian noise sampled from $N(0,\sigma^2)$, scaled by noise rate;
(b) Uniform noise drawn from $Unif(-\delta,\delta)$, where $\delta$ is proportional to noise rate;
(c) Mask noise that randomly zeros out features according to the noise rate.
 \item Structure Noise:
On Cora-CA, we perturb hypergraph structure by simultaneously removing existing hyperedges and adding fake hyperedges (random node subsets) at the specified noise rate.
\end{itemize}
\begin{figure}[htb]
\centering  
\subfigure[]
{\label{fig:lhid}
\includegraphics[width=7cm,height =5cm]{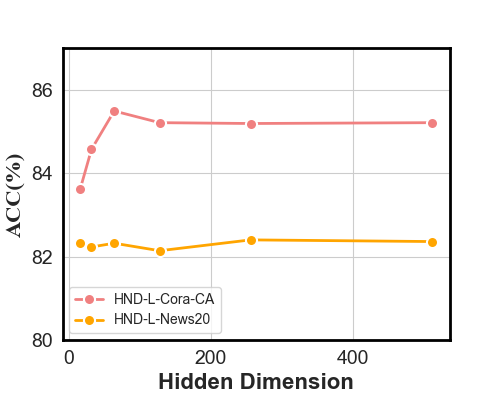}}
\subfigure[]
{\label{fig:nlhid}
\includegraphics[width=7cm,height =5cm]{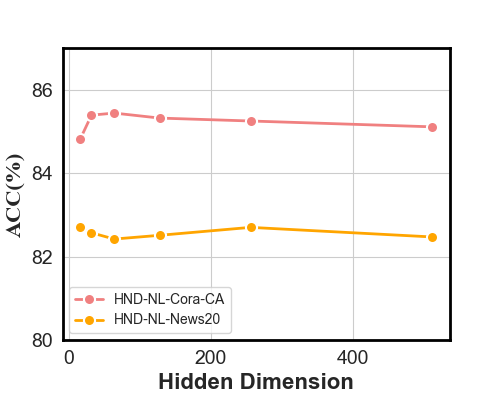}}
\subfigure[]
{\label{fig:lt}
\includegraphics[width=7cm,height =5cm]{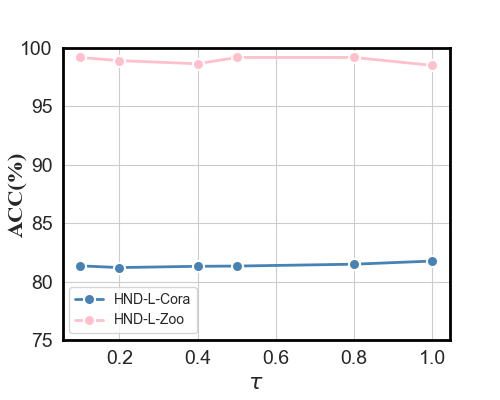}}
\subfigure[]
{\label{fig:nlt}
\includegraphics[width=7cm,height =5cm]{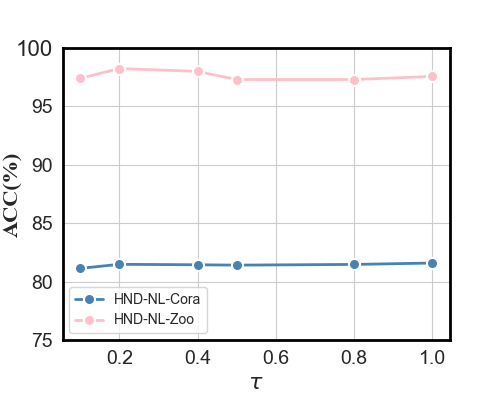}}
\caption{(a) shows HND-L accuracy on Cora-CA and News20 with various hidden dimensions. (b) shows HND-NL accuracy on Cora-CA and News20 with various hidden dimensions. (c) shows HND-L accuracy on Cora and Zoo with various $\tau$. (d) shows HND-NL accuracy on Cora and Zoo with various $\tau$.}
\label{para}
\end{figure}
As demonstrated in Figure \ref{noise}, our method consistently outperforms baseline approaches across all noise conditions. Under feature-level perturbations (Figures \ref{fig:gaussian}, \ref{fig:uniform}, \ref{fig:mask}), both model variants maintain robust performance on the Citeseer dataset, showing comparable stability regardless of noise type - aussian, uniform, or mask-based corruption. Interestingly, all evaluated models exhibit greater resilience to mask noise compared to other feature perturbations, likely due to the inherent information preservation characteristic of binary feature masking.

The structural noise analysis (Figure \ref{fig:stru}) reveals similar advantages on the CORA-CA dataset. Our method maintains superior performance at all noise levels (0.1-0.4), with both HND-NL and HDN-L variants showing significantly better results than baseline methods. While all models experience performance degradation with increasing noise rates, the observed decline remains remarkably limited - particularly for our proposed method - suggesting that structural perturbations have relatively minor impact on model effectiveness for this dataset. This robustness can be attributed to the inherent stability of hypergraph-based representations against edge-level modifications.
\begin{figure}[htb]
\centering  
\subfigure[Citeseer(t=0)]
{\label{fig:citeseerv}
\includegraphics[width=4.5cm,height =3.6cm]
{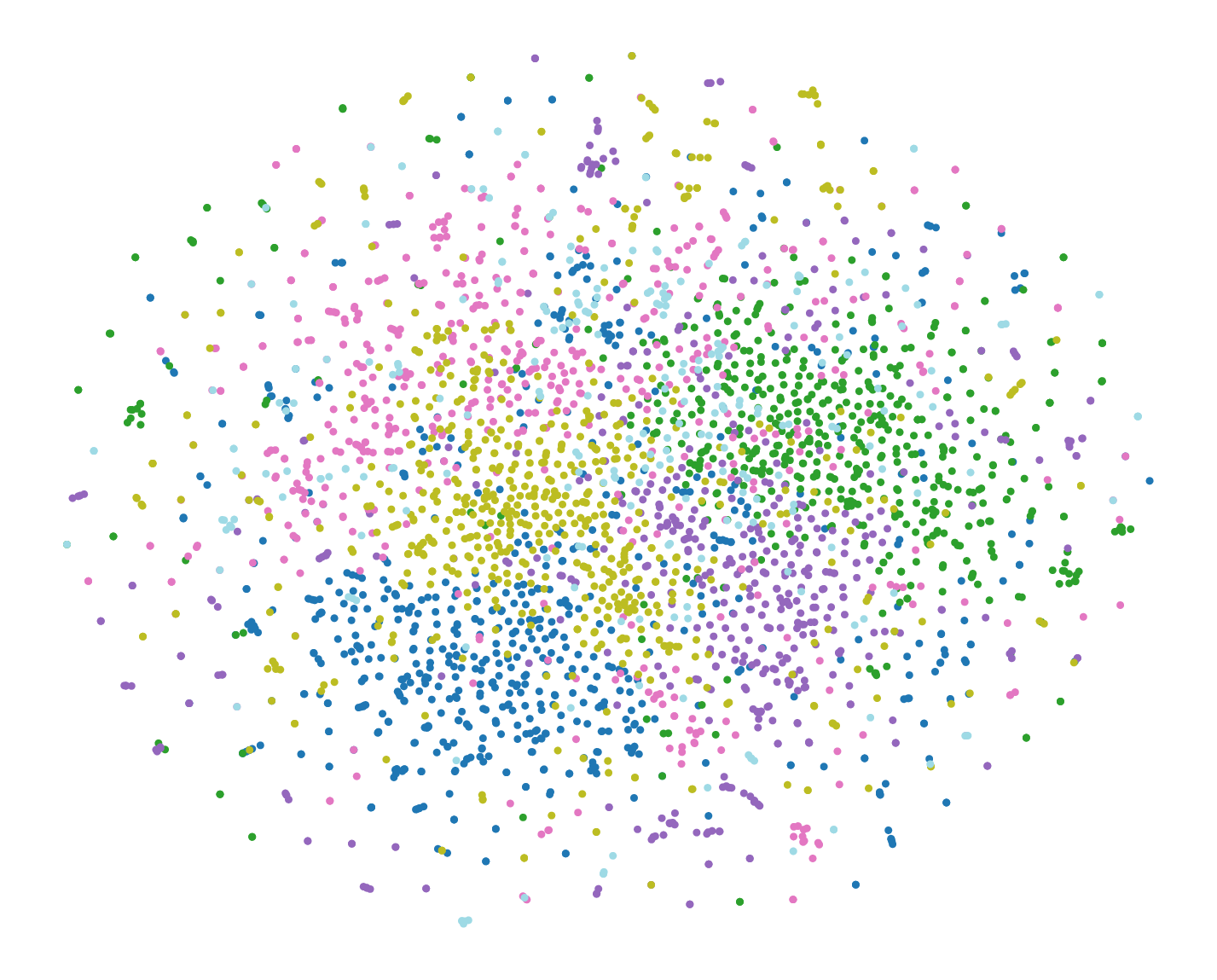}}
\subfigure[Citeseer(t=2)]
{\label{fig:citeseerv2}
\includegraphics[width=4.5cm,height =3.6cm]
{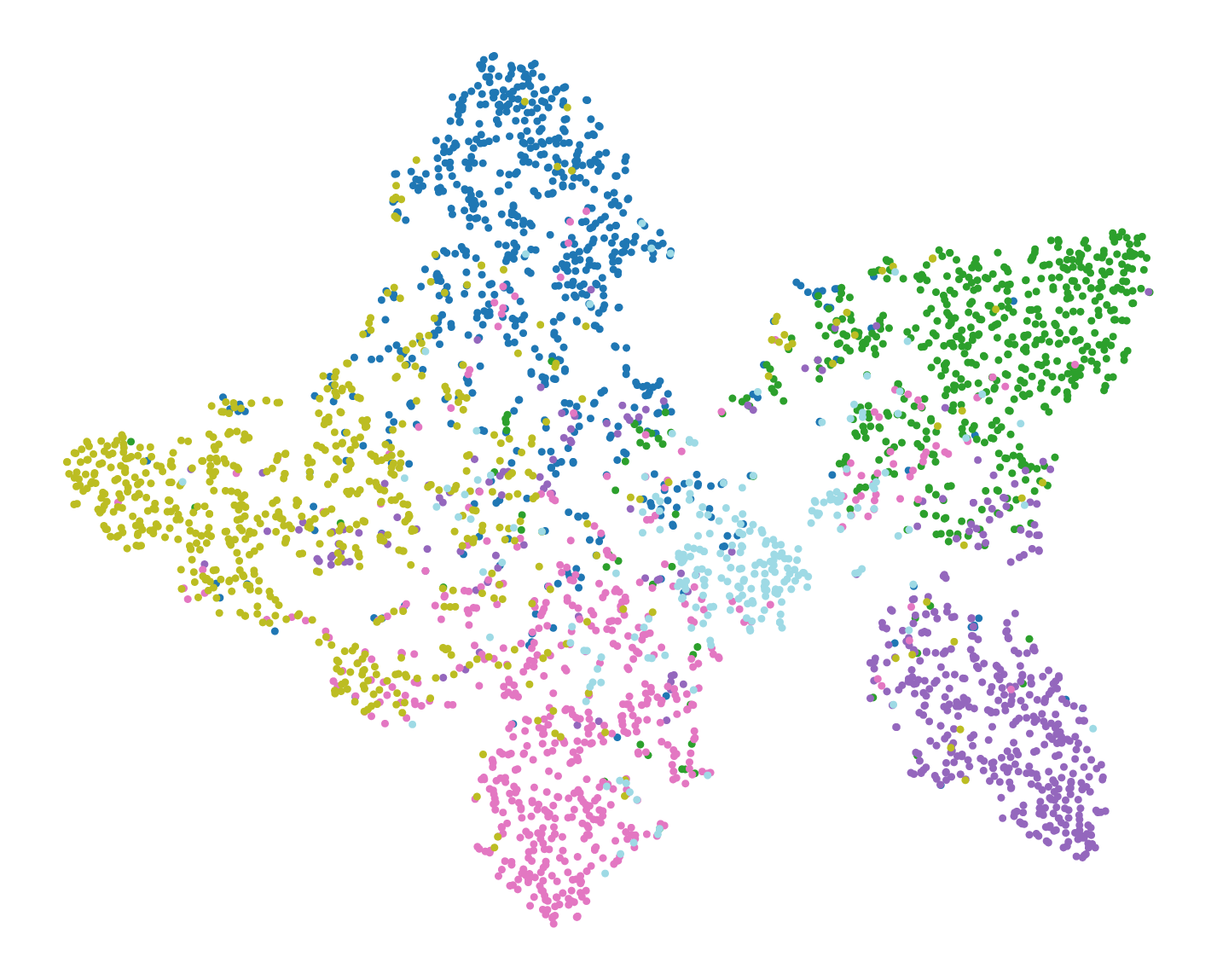}}
\subfigure[Citeseer(t=4)]
{\label{fig:citeseerv4}
\includegraphics[width=4.5cm,height =3.6cm]
{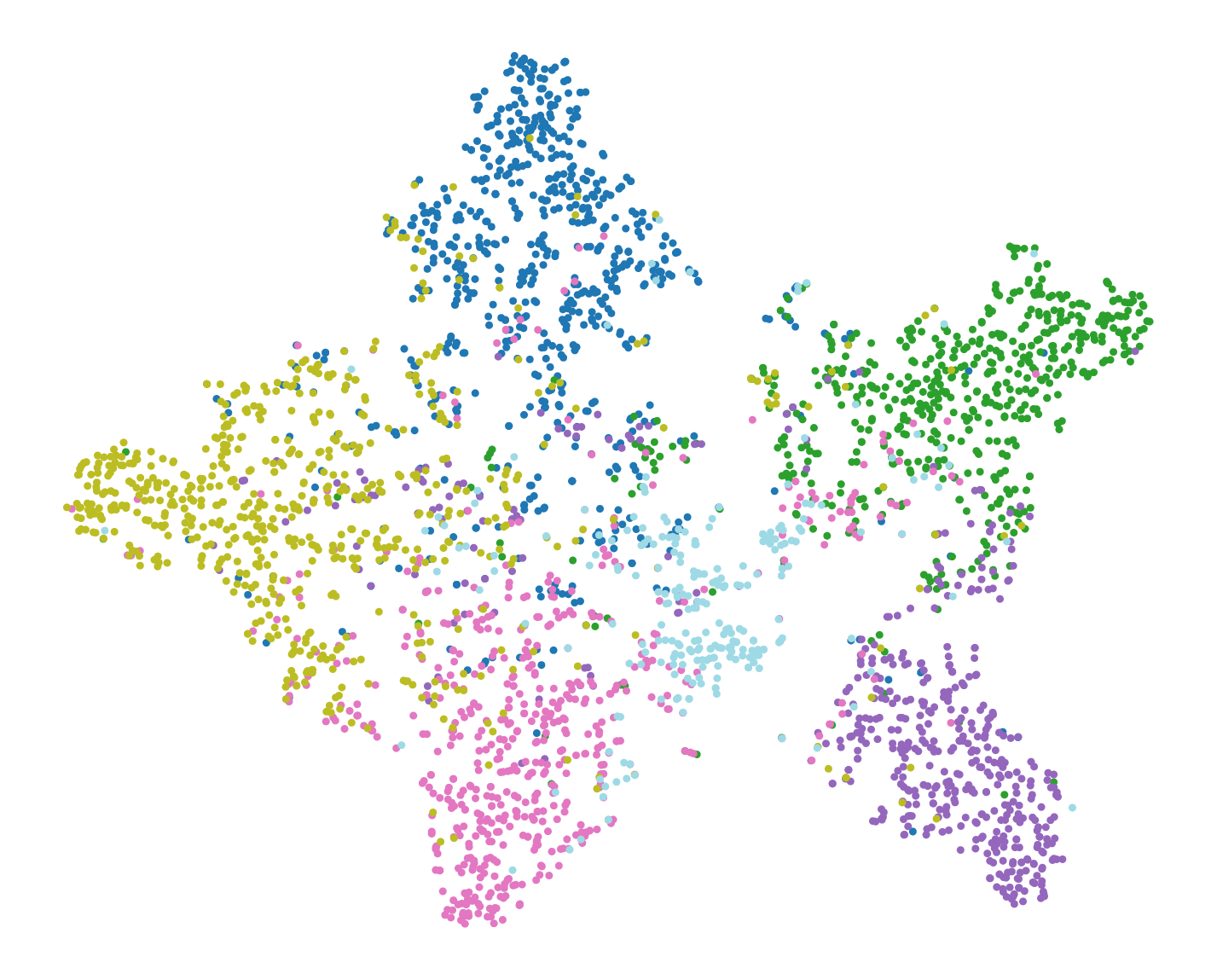}}
\subfigure[Pubmed(t=0)]
{\label{fig:pubmedv}
\includegraphics[width=4.5cm,height =3.6cm]
{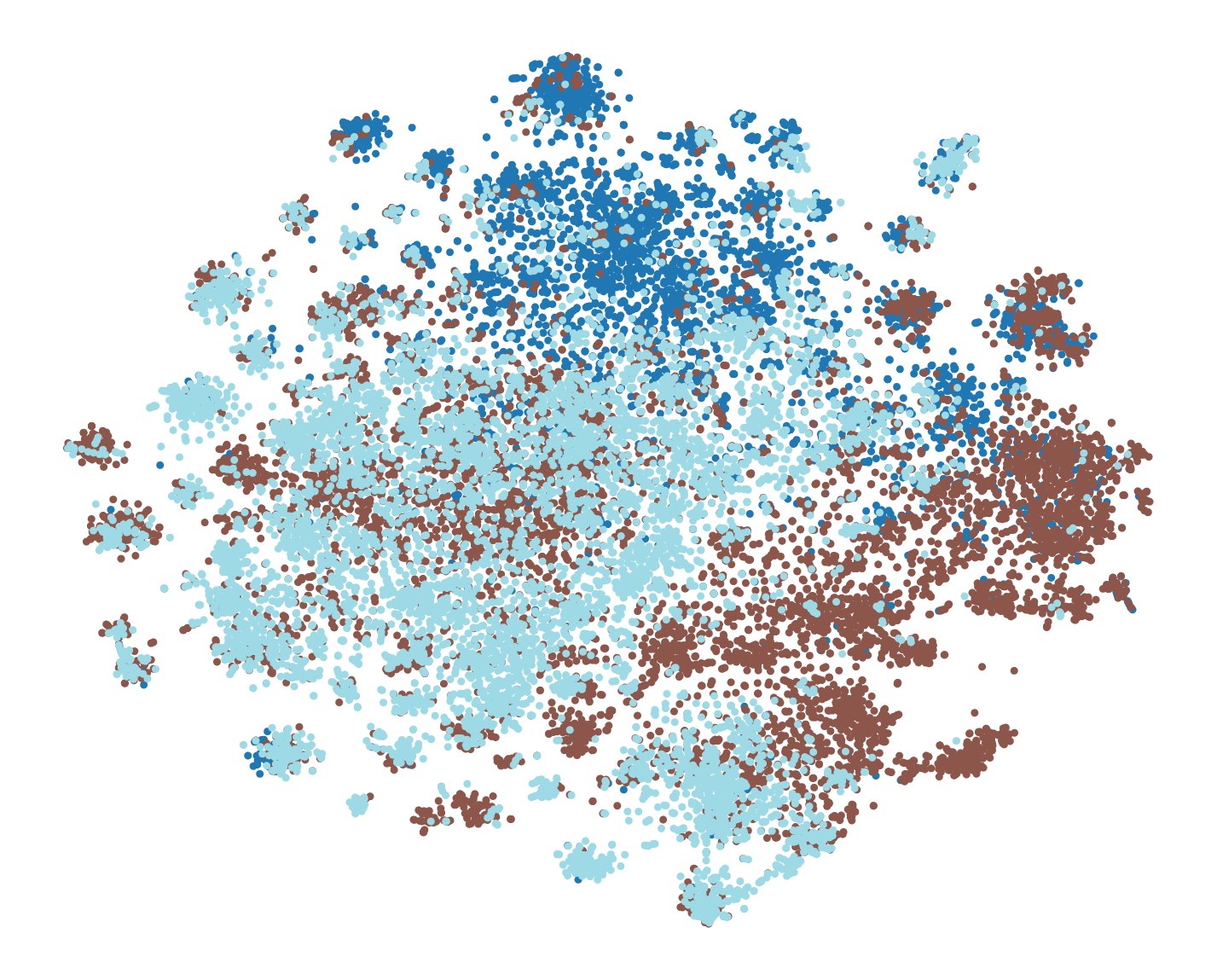}}
\subfigure[Pubmed(t=2)]
{\label{fig:pubmedv2}
\includegraphics[width=4.5cm,height =3.6cm]
{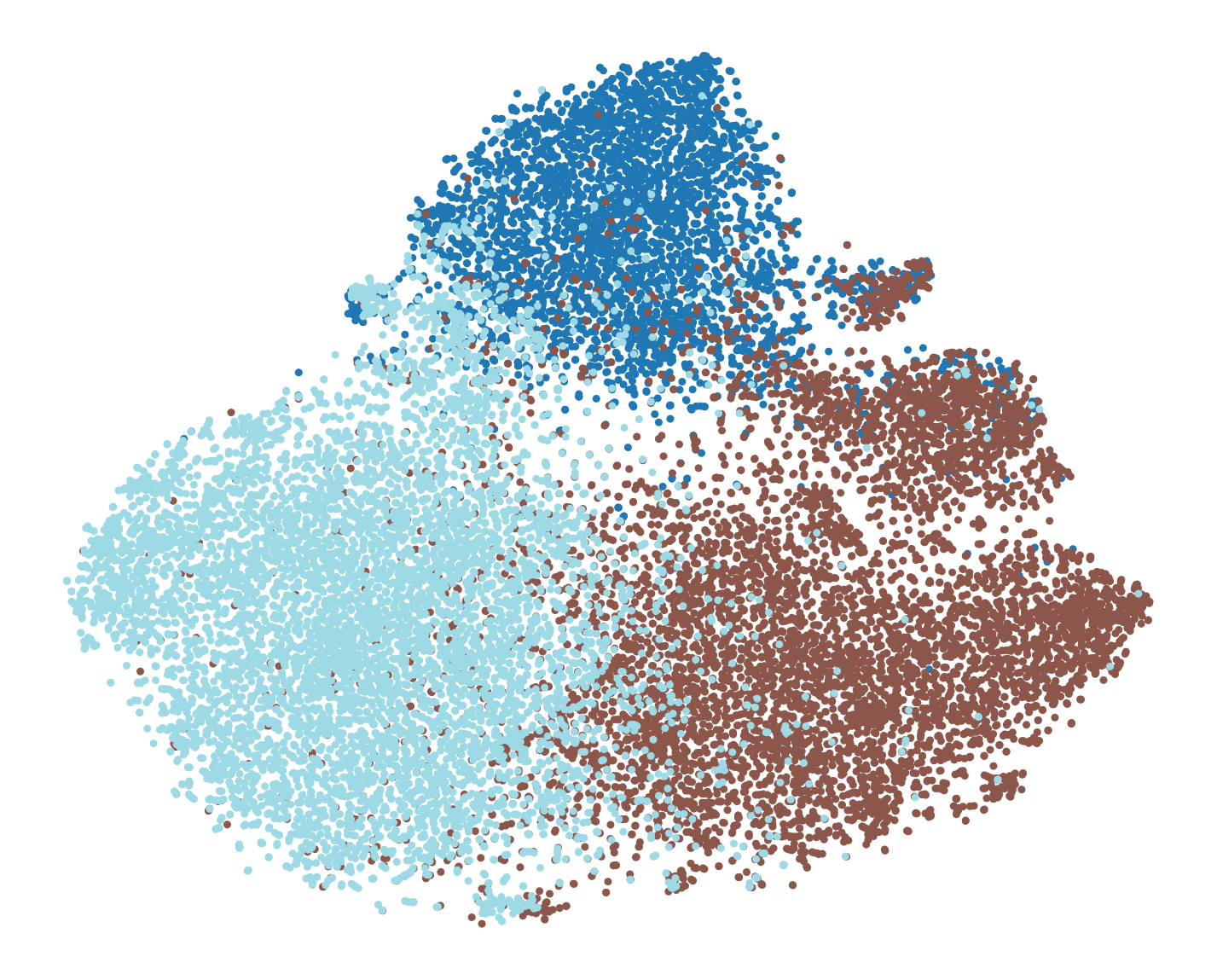}}
\subfigure[Pubmed(t=4)]
{\label{fig:pubmedv4}
\includegraphics[width=4.5cm,height =3.6cm]
{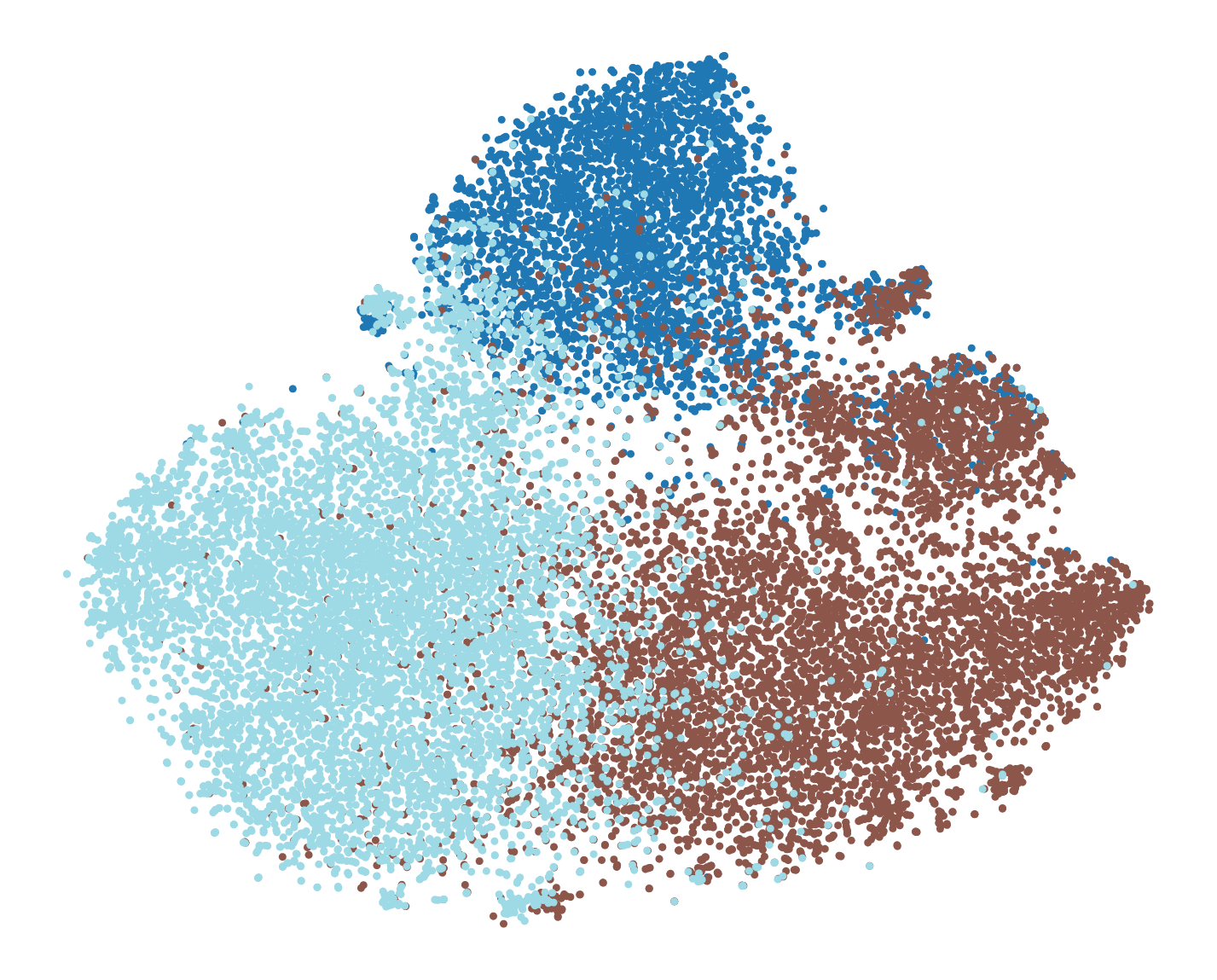}}
\caption{(a)(b)(c) show the visualization of Citeseer node features at $t = 0$, $t = 2$ and $t = 4$. (d)(e)(f) show the visualization of Pubmed node features at $t = 0$, $t = 2$ and $t = 4$. }
\label{vis}
\end{figure}

\subsection{Parameter sensitivity analysis}
To evaluate the robustness of the HND with respect to key hyperparameters, we conducted a comprehensive sensitivity analysis on two critical parameters: hidden layer dimension and step size $\tau$. Specifically, we examined model performance under various settings of hidden layer dimensions $\{16, 32, 64, 128, 256, 512\}$ and step sizes $\{0.1, 0.2, 0.4, 0.6, 0.8, 1\}$.

For the analysis of hidden layer dimension, we selected two representative datasets: Cora-CA and News20. For the step size $\tau$, the experiments were conducted on the Zoo and Cora datasets. As illustrated in Figure \ref{para}, the model performance remains relatively stable across the different settings of both hyperparameters. This indicates that the proposed model is not sensitive to the variations in hidden layer dimension and step size, demonstrating strong robustness, good generalization ability, and stable performance.

\subsection{Visualization}
To better understand the dynamic evolution of node representations, we performed a feature visualization analysis on the Citeseer and Pubmed datasets. Specifically, we recorded the node features at different time steps during the ODE integration process, namely at $t = 0$, $t = 2$, and $t = 4$. These high-dimensional features were projected into a 2D space using a dimensionality reduction technique for visualization. In the plots, each node is colored according to its class label.

As shown in Figures \ref{fig:citeseerv},\ref{fig:citeseerv2},\ref{fig:citeseerv4} for Citeseer and Figures \ref{fig:pubmedv},\ref{fig:pubmedv2},\ref{fig:pubmedv4} for Pubmed, we observe that as time progresses, the node embeddings become increasingly well-clustered. This indicates that the model effectively learns to separate nodes of different classes over time, and the representations evolve towards more discriminative structures as the integration proceeds.

\section{Conclusions}
In this work, we proposed HND, a principled and expressive framework that bridges nonlinear PDEs and HGNNs. Grounded in a continuous-time diffusion equation defined via hypergraph gradient and divergence operators, HND introduces a learnable, structure-aware modulation matrix that governs feature-adaptive and anisotropic diffusion across hyperedge–node pairs. This PDE-based perspective enables the development of novel, stable, and interpretable neural models through flexible numerical schemes, including fixed-step methods such as explicit and implicit Euler methods and multi-step methods, as well as adaptive-step integrators.

We established theoretical guarantees for the HND framework, including energy dissipation, maximum principle, and numerical stability, which collectively underpin the robustness and stability of the model. Extensive experiments on diverse benchmark datasets demonstrate that HND achieves competitive performance.

Our work lays the foundation for a new class of hypergraph neural models grounded in continuous dynamics. Future directions include extending HND to stochastic diffusion settings for uncertainty estimation, incorporating geometric priors such as sheaves and manifolds, and exploring applications in generative modeling and spatiotemporal hypergraph forecasting.



\section*{Acknowledgments}
This work was supported by the National Natural Science Foundation of China under Grant No.12231018.


\newpage

\appendix
\section{Proof}\label{app:proposition}
\subsection{Proof of Proposition \ref{P1}}
\begin{proof}
Using the definition of the inner product on $L(\mathcal{E}, \mathcal{V})$, we compute:
\begin{equation}
\langle \nabla f, g \rangle_{L(\mathcal{E}, \mathcal{V})} = \sum_{(e,v) \in \mathcal{I}} w_e(\nabla f)(e,v) \cdot g(e,v).
\end{equation}
Substituting the definition of $\nabla f$ gives:
\begin{equation}
= \sum_{(e,v)} w_e \left( \frac{f(v)}{\sqrt{d_v}} - \frac{1}{|e|} \sum_{u \in e}\frac{f(u)}{\sqrt{d_u}} \right) g(e,v).
\end{equation}
We split the expression into two terms:
\begin{equation}
= \sum_{(e,v)} \frac{f(v)}{\sqrt{d_v}} w_e g(e,v) - \sum_{(e,v)} \left( \sum_{u \in e} \frac{f(u)}{\sqrt{{d_u}}} \right)\frac{w_e}{|e|} g(e,v).
\end{equation}
Expanding the summation yields:
\begin{equation}
= \sum_{v \in \mathcal{V}}\frac{f(v)}{\sqrt{d_v}} \left( \sum_{e \ni v} w_e g(e,v) \right) - \sum_{v \in \mathcal{V}} \frac{f(v)}{\sqrt{d_v}} \left( \sum_{e \ni v}  \sum_{u \in e} \frac{w_e}{|e|}g(e,u) \right).
\end{equation}
Combining both expressions, we obtain:
\begin{equation}
\langle \nabla f, g \rangle_{L(\mathcal{E}, \mathcal{V})} = \sum_{v \in \mathcal{V}} f(v)  \sum_{e \ni v}\frac{w_e}{\sqrt{d_v}} \left(g(e,v) -  \sum_{u \in e} \frac{1}{|e|}g(e,u) \right) .
\end{equation}
This matches the inner product $\langle f, \operatorname{div} g \rangle_{L(\mathcal{V})}$ by the definition of the divergence operator. Therefore,
\begin{equation}
\langle \nabla f, g \rangle_{L(\mathcal{E}, \mathcal{V})} = \langle f, \operatorname{div} g \rangle_{L(\mathcal{V})}.
\end{equation}
\end{proof}
\subsection{Proof of Proposition \ref{P2}}
\begin{proof}
We first show symmetry. Let $f, h \in L(\mathcal{V})$. Then:
\begin{equation}
\langle \Delta f, h \rangle_{L(\mathcal{V})} 
= \langle \operatorname{div}(\nabla f), h \rangle_{L(\mathcal{V})} 
= \langle \nabla f, \nabla h \rangle_{L(\mathcal{E}, \mathcal{V})},
\end{equation}
where the last equality follows from the adjointness of divergence and gradient.

By symmetry of the inner product on \(L(\mathcal{E}, \mathcal{V})\), we have:
\begin{equation}
\langle \nabla f, \nabla h \rangle = \langle \nabla h, \nabla f \rangle = \langle \Delta h, f \rangle,
\end{equation}
hence \(\langle \Delta f, h \rangle = \langle f, \Delta h \rangle\), which proves that \(\Delta\) is self-adjoint.

Now we prove positive semi-definiteness. For any \(f \in L(\mathcal{V})\),
\begin{equation}
\langle \Delta f, f \rangle = \langle \nabla f, \nabla f \rangle = \| \nabla f \|^2 \geq 0,
\end{equation}
since the inner product induces a norm on the Hilbert space \(L(\mathcal{E}, \mathcal{V})\). Equality holds if and only if \(\nabla f = 0\), i.e., \(f\) is constant on every hyperedge.

Thus, \(\Delta\) is symmetric and positive semi-definite.
\end{proof}
\subsection{Proof of Proposition \ref{P3}}\label{proof3}
\begin{proof}
We adopt the gradient operator and divergence operator definition:
\begin{equation}
\begin{aligned}
(\nabla f)(e,v): &=   \frac{f(v)}{\sqrt{d_v}} - \frac{1}{|e|} \sum_{u \in e} \frac{f(u)}{\sqrt{d_u}} ,\\
(\operatorname{div} g)(v) :&= \sum_{e \ni v}\frac{w_e}{\sqrt{d_v}}\left(  g(e, v) - \frac{1}{|e|} \sum_{u \in e} g(e, u)\right),
\end{aligned}
\end{equation}
which leads to the gradient matrix and divergence matrix representation:
\begin{equation}\label{PQ}
\begin{aligned}
P &= (B-C)D_v^{-1/2} \in \mathbb R^{N \times n},\\
Q &= D_v^{-1/2}(B-C)^\top S\in \mathbb R^{n \times N},
\end{aligned}
\end{equation}
where $B,C$ maps node functions to each hyperedge, and $S = \operatorname{diag}\!\bigl(w_e\bigr)
      \in \mathbb R^{N \times N}$
\begin{equation}
\begin{aligned}
B_{(e,v),u} &=
\begin{cases}
1,& u = v,\\
0,& \text{otherwise},
\end{cases}
\qquad B \in \mathbb R^{N \times n},\mbox{ and }
C_{(e,v),u} =
\begin{cases}
\displaystyle\frac{1}{|e|},& u \in e,\\
0,& u \notin e,
\end{cases}
\qquad C \in \mathbb R^{N \times n}.
\end{aligned}
\end{equation}

Let us compute:
\begin{equation}
QP = D_v^{-1/2}(B-C)^\top S (B-C)D_v^{-1/2}.
\end{equation}
Expanding this product yields:
\begin{equation}
(B-C)^\top S (B-C) = B^\top S B - B^\top S C - C^\top S B + C^\top S C,
\end{equation}
which simplifies to:
\begin{equation}
(B-C)^\top S (B-C) = W_1 - W_2 - W_2^\top + W_3,
\end{equation}
where:
For any \(v,u\in\mathcal V\), 
\begin{equation} 
(W_1)_{v,v} = \sum_{e\ni v}w_e = d_v,\quad (W_1)_{v,u}=0\;(u\neq v),
\end{equation}
hence \(W_1=D_v\), 
\begin{equation} 
(W_2)_{v,u} = \sum_{e\ni v,u}\frac{w_e}{|e|} = (H W_e D_e^{-1} H^\top)_{v,u}, 
\end{equation}
so \(W_2= H W_e D_e^{-1} H^\top\). By symmetry, \(W_2^\top = H W_e D_e^{-1} H^\top \). For \(v,u\in e\) there are exactly \(|e|\) rows \((e,x)\) with
\(C_{(e,x),v}=C_{(e,x),u}=1/|e|\); thus 
\begin{equation} 
(W_3)_{v,u} = (C^\top W C)_{v,u}=\sum_{e\ni v,u} \frac{|e|\;w_e}{|e|^2}
=(H W_e D_e^{-1} H^\top )_{v,u}, 
\end{equation}
and again $W_3= H W_e D_e^{-1} H^\top $. because \( W_2 + W_2^\top = 2 H W_e D_e^{-1} H^\top  \), and \( W_3 = H W_e D_e^{-1} H^\top  \).

So,
\begin{equation}
\begin{aligned}
&(B-C)^\top S (B-C) = D_v-H W_e D_e^{-1} H^\top\\
QP = D_v^{-1/2}(&D_v-H W_e D_e^{-1} H^\top)D_v^{-1/2}=I - D_v^{-1/2}H W_e D_e^{-1} H^\top D_v^{-1/2},
\end{aligned}
\end{equation}
Thus, the Laplacian matrix \( \mathcal{L} = QP \) = \( I - D_v^{-1/2} H W_e  D_e^{-1} H^\top D_v^{-1/2} \), as claimed.
\end{proof}

\subsection{Proof of Proposition \ref{P6}}

\begin{proof}
We analyze the node-wise evolution of \( \frac{\mathbf{x}_v(t)}{\sqrt{d_v}} \). From the diffusion equation, we have:

\begin{equation}\label{eqxv}
\begin{aligned}
\frac{\partial \frac{\mathbf{x}_v(t)}{\sqrt{d_v}}}{\partial t} = -& \sum_{e \ni v} \frac{w_e}{d_v}\Bigg[ a(\mathbf{x}_e(t),\mathbf{x}_v(t)) \left( \frac{\mathbf{x}_v(t)}{\sqrt{d_v}} - \frac{1}{|e|} \sum_{u \in e} \frac{\mathbf{x}_u(t)}{\sqrt{d_u}} \right) \\
-& \frac{1}{|e|} \sum_{u \in e} a(\mathbf{x}_e(t),\mathbf{x}_u(t)) \left( \frac{\mathbf{x}_u(t)}{\sqrt{d_u}} - \frac{1}{|e|} \sum_{w \in e} \frac{\mathbf{x}_w(t)}{\sqrt{d_w}} \right) \Bigg]
\end{aligned}
\end{equation}

We aim to show that the maximum and minimum of \( \frac{\mathbf{x}_v(t)}{\sqrt{d_v}} \) are bounded by their initial values, i.e.,

\begin{equation}
\min_{v \in \mathcal{V}} \frac{\mathbf{x}_v(0)}{\sqrt{d_v}} \leq \frac{\mathbf{x}_v(t)}{\sqrt{d_v}} \leq \max_{v \in \mathcal{V}} \frac{\mathbf{x}_v(0)}{\sqrt{d_v}} \quad \text{for all} \quad t \geq 0
\end{equation}

Define the maximum and minimum of \( \frac{\mathbf{x}_v(t)}{\sqrt{d_v}} \) as:

\begin{equation}
M(t) = \max_{v \in \mathcal{V}} \frac{\mathbf{x}_v(t)}{\sqrt{d_v}}, \quad m(t) = \min_{v \in \mathcal{V}} \frac{\mathbf{x}_v(t)}{\sqrt{d_v}}
\end{equation}

Our goal is to prove that the maximum and minimum values of \( \frac{\mathbf{x}_v(t)}{\sqrt{d_v}} \) at any time \( t \) are bounded by the initial values:

\begin{equation}
m(0) \leq m(t) \leq M(t) \leq M(0) \quad \text{for all} \quad t \geq 0
\end{equation}

To understand how the maximum and minimum evolve, we differentiate \( M(t) \) and \( m(t) \):

For the maximum value \( M(t) \), we have:

\begin{equation}
\frac{dM(t)}{dt} = \max_{v \in \mathcal{V}} \frac{\partial \frac{\mathbf{x}_v(t)}{\sqrt{d_v}}}{\partial t}
\end{equation}

For the minimum value \( m(t) \), we have:

\begin{equation}
\frac{dm(t)}{dt} = \min_{v \in \mathcal{V}} \frac{\partial \frac{\mathbf{x}_v(t)}{\sqrt{d_v}}}{\partial t}
\end{equation}

The right-hand side of the equation involves two terms:

1. The first term represents the difference between the normalized state of node \( v \) and the average of its neighboring nodes' normalized states.
2. The second term represents the influence of node \( u \) on its neighbors.

Since the Eq.(\ref{eqxv}) involves weighted averages of neighboring nodes' normalized states, the evolution of \( \frac{\mathbf{x}_v(t)}{\sqrt{d_v}} \) is smooth. This ensures that \( \frac{\mathbf{x}_v(t)}{\sqrt{d_v}} \) changes gradually over time and does not experience abrupt jumps. Let \( M(t) = \frac{\mathbf{x}_{v^*}(t)}{\sqrt{d_v}} \), where \( v^* \) is the node that maximizes \( \frac{\mathbf{x}_v(t)}{\sqrt{d_v}} \). We need to show that \( M(t) \leq M(0) \). 

The rate of change of \( \frac{\mathbf{x}_{v^*}(t)}{\sqrt{d_v}} \) is determined by the Eq.(\ref{eqxv}). Since \( v^* \) has the maximum value, its feature \( \frac{\mathbf{x}_{v^*}(t)}{\sqrt{d_v}} \) is influenced by the features of its neighboring nodes. The right-hand side involves a weighted average of neighboring nodes' normalized states. This ensures that the feature \( \frac{\mathbf{x}_{v^*}(t)}{\sqrt{d_v}} \) will be updated based on its neighbors' states, but it cannot exceed the initial maximum value \( M(0) \) because of the smooth nature of the update. Thus, \( M(t) \leq M(0) \) for all \( t \geq 0 \).

Similarly, let \( m(t) = \frac{\mathbf{x}_{v_*}(t)}{\sqrt{d_v}} \), where \( v_* \) is the node that minimizes \( \frac{\mathbf{x}_v(t)}{\sqrt{d_v}} \). We need to show that \( m(t) \geq m(0) \).

The rate of change of \( \frac{\mathbf{x}_{v_*}(t)}{\sqrt{d_v}} \) is determined by the Eq.(\ref{eqxv}). Since \( v_* \) has the minimum value, its feature \( \frac{\mathbf{x}_{v_*}(t)}{\sqrt{d_v}} \) is influenced by the features of its neighboring nodes. Again, the Eq.(\ref{eqxv}) involves a weighted average of neighboring nodes' normalized states, ensuring that \( \frac{\mathbf{x}_{v_*}(t)}{\sqrt{d_v}} \) cannot decrease below the initial minimum value \( m(0) \). Thus, \( m(t) \geq m(0) \) for all \( t \geq 0 \).

From the above steps, we have shown that both the maximum value \( M(t) \) and the minimum value \( m(t) \) of \( \frac{\mathbf{x}_v(t)}{\sqrt{d_v}} \) remain bounded by their initial values:

\begin{equation}
m(0) \leq m(t) \leq M(t) \leq M(0) \quad \text{for all} \quad t \geq 0
\end{equation}

Thus, we have proven the Maximum Principle for the diffusion equation in Eq.(\ref{eqxv}), which shows that the node states \( \frac{\mathbf{x}_v(t)}{\sqrt{d_v}} \) remain within the bounds set by the initial conditions.
\end{proof}

\subsection{Proof of Proposition \ref{P7}}
\begin{proof}
The explicit Euler method updates the state \( \mathbf{x}^{(k+1)} \) based on the current state \( \mathbf{x}^{(k)} \). The discrete form of the update is given by:
\begin{equation}
\mathbf{x}^{(k+1)} = \mathbf{x}^{(k)} - \tau \cdot G^\top \mathbf{A}(\mathbf{x}^{(k)}) G \mathbf{x}^{(k)},
\end{equation}
where $\tau$ is the time step.

Next, we expand the squared norm of the updated state:
\begin{equation}
\| \mathbf{x}^{(k+1)} \|^2 = \left( \mathbf{x}^{(k)} - \tau G^\top \mathbf{A}(\mathbf{x}^{(k)}) G \mathbf{x}^{(k)} \right)^\top \left( \mathbf{x}^{(k)} - \tau G^\top \mathbf{A}(\mathbf{x}^{(k)}) G \mathbf{x}^{(k)} \right).
\end{equation}
This simplifies to:
\begin{equation}
\| \mathbf{x}^{(k+1)} \|^2 = \| \mathbf{x}^{(k)} \|^2 - 2\tau \langle \mathbf{x}^{(k)}, G^\top \mathbf{A}(\mathbf{x}^{(k)}) G \mathbf{x}^{(k)} \rangle + \tau^2 \| G^\top \mathbf{A}(\mathbf{x}^{(k)}) G \mathbf{x}^{(k)} \|^2.
\end{equation}
For stability, we require that the norm of the updated state is not greater than the norm of the previous state, i.e.,
\begin{equation}
\| \mathbf{x}^{(k+1)} \|^2 \leq \| \mathbf{x}^{(k)} \|^2.
\end{equation}
This implies:
\begin{equation}
- 2 \tau \langle \mathbf{x}^{(k)}, G^\top \mathbf{A}(\mathbf{x}^{(k)}) G \mathbf{x}^{(k)} \rangle + \tau^2 \| G^\top \mathbf{A}(\mathbf{x}^{(k)}) G \mathbf{x}^{(k)} \|^2 \leq 0.
\end{equation}
We can now rewrite this expression as:
\begin{equation}
\tau \leq \frac{2 \langle \mathbf{x}^{(k)}, G^\top \mathbf{A}(\mathbf{x}^{(k)}) G \mathbf{x}^{(k)} \rangle}{\| G^\top \mathbf{A}(\mathbf{x}^{(k)}) G \mathbf{x}^{(k)} \|^2}.
\end{equation}
Define $M=G^\top \mathbf{A}(\mathbf{x}^{(k)}) G$, the expression becomes:
\begin{equation}
\tau\leq \frac{2 {x^{(k)}}^\top Mx^{(k)}}{{x^{(k)}}^\top M^2x^{(k)}},
\end{equation}
This is a Rayleigh quotient for the matrix $M$. As \( G^\top G \) has eigenvalues in \( [0, 2] \), and \( \mathbf{A}(\mathbf{x}(t)) \) is a diagonal matrix with entries constrained by \( \sum_{e\ni v}a(\mathbf{x}_e(t),\mathbf{x}_v(t))=1 \), ensuring $a(\mathbf{x}_e(t),\mathbf{x}_v(t))\leq 1$. Consequently, the eigenvalues of $M$ lie in the interval \([0,2]\).

Since the largest eigenvalue of $M$, denoted $\lambda_{\text{max}}(M)$, satisfies $\lambda_{\text{max}}(M)\leq2$, the Rayleigh quotient gives the following upper bound for $\tau$:
\begin{equation}
\tau \leq  \frac{2}{\lambda_{max}}.
\end{equation}
We conclude that $\tau\leq 1$.

Thus, for stability, the explicit Euler method for the HDE is stable if \( \tau \leq 1 \).
\end{proof}
\subsection{Proof of Proposition \ref{P8}}
\begin{proof}
We consider the discretized form of the HDE using the implicit Euler method:
\begin{equation}
\mathbf{x}^{(k+1)} = \mathbf{x}^{(k)} - \tau \cdot G^\top \mathbf{A}(\mathbf{x}^{(k+1)}) G \mathbf{x}^{(k+1)},
\end{equation}
where \( \tau \) is the time step. To prove the stability, we will show that the solution does not grow unbounded for any choice of \( \tau > 0 \).

Rearranging the equation as follows to isolate \( \mathbf{x}^{(k+1)} \):
\begin{equation}
\mathbf{x}^{(k+1)} + \tau \cdot G^\top \mathbf{A}(\mathbf{x}^{(k+1)}) G \mathbf{x}^{(k+1)} = \mathbf{x}^{(k)}.
\end{equation}
This can be rewritten as:
\begin{equation}
\left( I + \tau \cdot G^\top \mathbf{A}(\mathbf{x}^{(k+1)}) G \right) \mathbf{x}^{(k+1)} = \mathbf{x}^{(k)}.
\end{equation}
Taking the norm of both sides of the equation:
\begin{equation}
\left\| \left( I + \tau \cdot G^\top \mathbf{A}(\mathbf{x}^{(k+1)}) G \right) \mathbf{x}^{(k+1)} \right\| = \| \mathbf{x}^{(k)} \|.
\end{equation}
The left-hand side can be expanded as:
\begin{equation}
\left\| \left( I + \tau \cdot G^\top \mathbf{A}(\mathbf{x}^{(k+1)}) G \right) \mathbf{x}^{(k+1)} \right\| = \| \mathbf{x}^{(k+1)} \| \cdot \left\| I + \tau \cdot G^\top \mathbf{A}(\mathbf{x}^{(k+1)}) G \right\|.
\end{equation}
Thus, the norm equation becomes:
\begin{equation}
\| \mathbf{x}^{(k+1)} \| \cdot \left\| I + \tau \cdot G^\top \mathbf{A}(\mathbf{x}^{(k+1)}) G \right\| = \| \mathbf{x}^{(k)} \|.
\end{equation}

Since \( \mathbf{A}(\mathbf{x}^{(k+1)}) \) is a diagonal matrix, we know that:
\begin{equation}
\left\| I + \tau \cdot G^\top \mathbf{A}(\mathbf{x}^{(k+1)}) G \right\| \leq 1 + \tau \cdot \lambda_{\max}(G^\top G) \cdot \lambda_{\max}(\mathbf{A}(\mathbf{x}^{(k+1)})).
\end{equation}
The largest eigenvalue of \( G^\top G \), \( \lambda_{\max}(G^\top G) \), is bounded, as we know \( \lambda_{\max}(G^\top G) \leq 2 \). The modulation matrix \( \mathbf{A}(\mathbf{x}^{(k+1)}) \) is also bounded by \( \lambda_{\max}(\mathbf{A}(\mathbf{x}^{(k+1)})) \leq 1 \). Thus, we have:
\begin{equation}
\left\| I + \tau \cdot G^\top \mathbf{A}(\mathbf{x}^{(k+1)}) G \right\| \leq 1 + 2 \tau.
\end{equation}
Now, we have the following inequality:
\begin{equation}
\| \mathbf{x}^{(k+1)} \| \cdot (1 + 2 \tau) = \| \mathbf{x}^{(k)} \|.
\end{equation}
Since \( \tau > 0 \), it follows that:
\begin{equation}
\| \mathbf{x}^{(k+1)} \| \leq \frac{1}{1 + 2 \tau} \| \mathbf{x}^{(k)} \|.
\end{equation}
This implies that the solution remains bounded and the norm does not increase over time.
\end{proof}

\section{Advanced Numerical Solvers and Models}\label{app:solvers and models}
\subsection{Multi-step and Adaptive Solvers}\label{app:solvers}
Runge–Kutta methods, while formally categorized as single-step methods due to their reliance only on the current time state, are in fact multi-stage schemes that compute intermediate estimates to achieve high-order accuracy. A classical example is the explicit fourth-order Runge–Kutta method (RK4), defined as:
\begin{equation} 
\begin{aligned} \mathbf{p}_1 = f(\mathbf{x}^{(k)}), \ \mathbf{p}_2 = &f\left(\mathbf{x}^{(k)} + \frac{\tau}{2} \mathbf{p}_1\right), \ \mathbf{p}_3 = f\left(\mathbf{x}^{(k)} + \frac{\tau}{2} \mathbf{p}_2\right), \ \mathbf{p}_4 = f\left(\mathbf{x}^{(k)} + \tau \mathbf{p}_3\right), \\ \mathbf{x}^{(k+1)} &= \mathbf{x}^{(k)} - \frac{\tau}{6}(\mathbf{p}_1 + 2\mathbf{p}_2 + 2\mathbf{p}_3 + \mathbf{p}_4). 
\end{aligned} 
\end{equation} 
where the nonlinear operator is defined as $f(\mathbf{x})=\mathcal{L}_{\text{NL}}(\mathbf{x})\mathbf{x}$. In contrast, general linear multi-step methods directly utilize several previous time steps to advance the solution. A generic $k$-step method can be written as: 
\begin{equation} \mathbf{x}^{(k+1)}+\sum\limits^k_{j=1}\alpha_j\mathbf{x}^{(k+1-j)} = \tau\sum\limits^k_{j=1}\beta_jf(\mathbf{x}^{(k+1-j)}), 
\end{equation} 
where $f(\mathbf{x})=\mathcal{L}_{\text{NL}}(\mathbf{x})\mathbf{x}$, $\{\alpha_j,\beta_j\}$ are method-specific coefficients that determine the accuracy and stability of the scheme.

The fourth-order Adams–Bashforth method, defined as: 
\begin{equation} 
\mathbf{x}^{(k+1)} = \mathbf{x}^{(k)}-\tau(\alpha_0f(\mathbf{x}^{(k)})+\alpha_1f(\mathbf{x}^{(k-1)})+\alpha_2f(\mathbf{x}^{(k-2)})+\alpha_3f(\mathbf{x}^{(k-3)})), \end{equation} 
with coefficients $\alpha_0=\frac{55}{24},\alpha_1=-\frac{59}{24},\alpha_2=\frac{37}{24},\alpha_3=-\frac{9}{24}$, where $f(\mathbf{x})=\mathcal{L}_{\text{NL}}(\mathbf{x})\mathbf{x}$.

For stiff or diffusion-dominated systems, implicit schemes like the Adams–Moulton method are preferred due to their superior stability. The fourth-order implicit Adams–Moulton method is given by:
\begin{equation} 
\mathbf{x}^{(k+1)} = \mathbf{x}^{(k)}-\tau(\beta_0f(\mathbf{x}^{(k+1)})+\beta_1f(\mathbf{x}^{(k)})+\beta_2f(\mathbf{x}^{(k-1)})+\beta_3f(\mathbf{x}^{(k-2)})), 
\end{equation} 
with coefficients $\beta_0=\frac{9}{24},\beta_1=\frac{19}{24},\beta_2=-\frac{5}{24},\beta_3=\frac{1}{24}$, where $f(\mathbf{x})=\mathcal{L}_{\text{NL}}(\mathbf{x})\mathbf{x}$. 

Adaptive step size control is another critical strategy for improving efficiency and reliability, especially when the dynamics of the system vary over time. In such methods, the time step $\tau$ is dynamically adjusted based on local error estimates. A widely used approach employs embedded RK pairs, which compute two approximations of different orders (say, $p$ and $p-1$) at each step. The local truncation error is estimated by: \begin{equation} 
\mbox{Error} = \|\mathbf{x}^{k+1}_{(p)}-\mathbf{x}^{k+1}_{(p-1)}\|, \end{equation} 
and the new time step is updated using the formula: 
\begin{equation} \tau_{new}=\tau\cdot\min\Big(\max \Big(\mbox{fac}_{\min},(\frac{\mbox{tol}}{\mbox{Error}})^{1/(p+1)}\Big),\mbox{fac}_{\max}\Big), 
\end{equation} 
where \(\text{tol}\) is a user-specified tolerance and $\mbox{fac}_{\min},\mbox{fac}_{\max}$ are safety factors used to prevent abrupt changes in step size.

\subsection{Multi-step and Adaptive Models}\label{app:models}

\paragraph{Adaptive Step HND.}
Rather than fixing the diffusion step size \( \tau \) across all layers, adaptive methods dynamically adjust \( \tau_{new} \) at each layer. This can be informed by signal variation, feature gradients, or local error estimation. Adaptive HND accommodates heterogeneous graph structures by varying diffusion intensity across space and time, enhancing both flexibility and performance.

\paragraph{Runge--Kutta HND.}
Higher-order RK schemes (e.g., RK4) can be adopted to achieve greater numerical accuracy. These involve multiple intermediate computations per layer, such as:
\[
\mathbf{X}^{(l+1)} = \mathbf{X}^{(l)} - \frac{\tau}{6}(\mathbf{p}_1 + 2\mathbf{p}_2 + 2\mathbf{p}_3 + \mathbf{p}_4),
\]
where each \( \mathbf{p}_i \) is computed based on the nonlinear diffusion operator \( \mathcal{L}_{\theta}(\mathbf{X}) \). RK-based layers improve precision in capturing transient diffusion dynamics and benefit applications requiring fine-grained information flow.

\paragraph{Multi-step HND.}
This variant leverages history from multiple previous layers:
\[
\mathbf{X}^{(l+1)} = \sum_{j=0}^{k} \alpha_j \mathbf{X}^{(l-j)} + \tau \sum_{j=0}^{k} \beta_j\mathcal{L}_{\theta}(\mathbf{X}^{(l-j)})\mathbf{X}^{(l-j)}.
\]
This enables the network to encode memory effects and temporal correlations across layers. When combined with suitable coefficient choices (e.g., Adams--Bashforth or Adams--Moulton), it allows for stable and high-order approximations of continuous dynamics.

\vskip 0.2in
\bibliography{sample}

\end{document}